%% file: multi_maml.tex
\documentclass[twoside]{article}

\usepackage[accepted]{aistats2021}
\usepackage[square]{natbib}

\usepackage{xcolor}
\usepackage[ruled,vlined]{algorithm2e}
\usepackage{import}
\usepackage{amsfonts}
\usepackage{bm}
\usepackage{mathtools}
\usepackage{amsthm,amsmath,amssymb,dsfont}
\usepackage{subcaption}
\usepackage{float}
\usepackage{hyperref}       % hyperlinks

\usepackage{array,multirow}

\usepackage{booktabs} \newcommand{\ra}[1]{\renewcommand{\arraystretch}{#1}}

\subimport{custom/}{macros.tex}

\begin{document}



\twocolumn[

\aistatstitle{Dif-MAML: Decentralized Multi-Agent Meta-Learning}

\aistatsauthor{ Mert Kayaalp \And Stefan Vlaski \And  Ali H. Sayed }

\aistatsaddress{} ]

\begin{abstract}
The objective of meta-learning is to exploit the knowledge obtained from observed tasks to improve adaptation to unseen tasks. As such, meta-learners are able to generalize better when they are trained with a larger number of observed tasks and with a larger amount of data per task. Given the amount of resources that are needed, it is generally difficult to expect the tasks, their respective data, and the necessary computational capacity to be available at a single central location. It is more natural to encounter situations where these resources are spread across several agents connected by some graph topology. The formalism of meta-learning is actually well-suited to this decentralized setting, where the learner would be able to benefit from information and computational power spread across the agents. Motivated by this observation, in this work, we propose a cooperative fully-decentralized multi-agent meta-learning algorithm, referred to as Diffusion-based MAML or Dif-MAML. Decentralized optimization algorithms are superior to centralized implementations in terms of scalability, avoidance of communication bottlenecks, and privacy guarantees. The work provides a detailed theoretical analysis to show that the proposed strategy allows a collection of agents to attain agreement at a linear rate and to converge to a stationary point of the \emph{aggregate} MAML objective even in non-convex environments. Simulation results illustrate the theoretical findings and the superior performance relative to the traditional non-cooperative setting.
\end{abstract}

\section{Introduction}
Training of highly expressive learning architectures, such as deep neural networks, requires large amounts of data in order to ensure high generalization performance. However, the generalization guarantees apply only to test data following the same distribution as the training data. Human intelligence, on the other hand, is characterized by a remarkable ability to leverage prior knowledge to accelerate adaptation to new tasks. This evident gap has motivated a growing number of works to pursue learning architectures that \emph{learn to learn} (see \citep{Hospedales2020MetaLearningIN} for a recent survey).

The work \citep{FinnAL17} proposed a model-agnostic meta-learning (MAML) approach, which is an initial parameter-transfer methodology where the goal is to learn a good ``{\em launch model}''. Several works have extended and/or analyzed this approach to great effect such as \citep{reptile,ProbabilisticMAML,Raghu2020Rapid,metasgd,FallahConvergence,Ji2020MultiStepMM,Nonconvex_online_ml,provablebalcan}. However, there does not appear to exist works that consider model agnostic meta-learning in a decentralized multi-agent setting. This setting is very natural to consider for meta-learning, where different agents can be assumed to have local meta-learners based on their own experiences. Interactions with neighbors can help infuse their models with new information and speed up adaptation to new tasks.

Decentralized multi-agent systems consist of a collection of agents with access to data and computational capabilities, and a graph topology that imposes constraints on peer-to-peer communications. In contrast to centralized architectures, which require some central aggregation of data, decentralized solutions rely solely on the diffusion of information over connected graphs through successive local aggregations over neighborhoods. While decentralized methods have been shown to be capable of matching the performance of centralized solutions \citep{Lian17,Sayed14}, the absence of a fusion center is advantageous in the presence of communication bottlenecks, and concerns around robustness or privacy. Decentralized settings are also well motivated by swarm intelligence or swarm robotics concepts where relatively simple agents (insects, machines, robots etc.) collaboratively form a more robust and complex system, one that is flexible and scalable \citep{Beni04,Sahin05}. Applications that can benefit from decentralized meta-learning algorithms include but are not limited to the following:
\begin{itemize}
    \item A robot swarm might be assigned to do environmental monitoring \citep{RobotsEnvMonitor}. The individual robots can share spatially and temporally dispersed data such as images or temperatures in order to learn better meta-models to adapt to new scenes. This teamwork is vital for circumstances where data collection is hard, such as natural disasters.
    \item Different hospitals or research groups can work on clinical risk prediction with limited patient health records \citep{clinicalriskpred} or drug discovery with small amount of data \citep{drugdiscovery}. The individual agents in this context will benefit from cooperation, while avoiding the need for a central hub in order to preserve the privacy of medical data.
    \item In some situations, it is advantageous to distribute a single agent problem over multiple agents. For example, training a MAML can be computationally demanding since it requires Hessian calculations \citep{FinnAL17}. In order to speed up the process, tasks can be divided into different workers or machines.
\end{itemize}

The contributions in this paper are three-fold:

\begin{itemize}
	\item By combining MAML with the diffusion strategy for decentralized stochastic optimization \citep{Sayed14}, we propose Diffusion-based Model-Agnostic Meta-Learning (Dif-MAML). The result is a decentralized algorithm for meta-learning over a collection of distributed agents, where each agent is provided with tasks stemming from potentially different task distributions. 
\item We establish that, despite the decentralized nature of the algorithm, all agents agree quickly on a common launch model, which subsequently converges to a stationary point of the \emph{aggregate} MAML objective over the task distribution across the network. This implies that Dif-MAML matches the performance of a centralized solution, which would have required central aggregation of data stemming from \emph{all tasks across the network}. In this way, agents will not only learn from locally observed tasks to accelerate future adaptation, but will also \emph{learn from each other}, and from tasks seen by the other agents.
\item We confirm through numerical experiments across a number of benchmark datasets that Dif-MAML outperforms the traditional non-cooperative solution and matches the performance of the centralized solution.
\end{itemize}

\textbf{Notation}. We denote random variables in bold. Single data points are denoted by small letters like \( x \) and batches of data are denoted by big calligraphic letters like \( \mathcal{X} \). To refer to a loss function evaluated at a batch \( \mathcal{X} \) with elements \( \left\{ x_{n} \right\}_{n=1}^{N} \), we use the notation  \( Q(w;\mathcal{X}) \triangleq \frac{1}{N} \sum_{n=1}^{N} Q(w; x_{n}) \). To denote expectation with respect to task-specific data, we use \( \e_{x^{(t)}} \) , where \( t \) corresponds to the task.

\subsection{Problem Formulation}

We consider a collection of \( K \) agents (e.g., robots, workers, machines, processors) where each agent \( k \) is provided with data stemming from tasks \( \mathcal{T}_k \). We denote the probability distribution over \( \mathcal{T}_k \) by \( \pi_{k} \), i.e., the probability of drawing task \( t \) from \( \mathcal{T}_k \) is \( \pi_{k} (t)\). In principle, for any particular task \( t \in \mathcal{T}_k \), each agent could learn a separate model \( {w_k^o}^{(t)} \) by solving:
\begin{equation}\label{eq:task_specific_expected}
    {w_k^o}^{(t)} \triangleq \argmin_{w \in \mathds{R}^M} J_k^{(t)}(w) \triangleq \argmin_{w \in \mathds{R}^M} \e_{x_{k}^{(t)}}  Q_k^{(t)}\Big(w; \boldsymbol{x}_{k}^{(t)}\Big) 
\end{equation}
where \( w \) denotes the model parametrization (such as the parameters of a neural network), while \( \boldsymbol{x}_{k}^{(t)} \) denotes the random data corresponding to task \( t \) observed at agent \( k \). The loss \( Q_k^{(t)}(w; \boldsymbol{x}_{k}^{(t)}) \) denotes the penalization encountered by \( w \) under the random data \( \boldsymbol{x}_{k}^{(t)} \), while \( J_k^{(t)}(w) \) represents the \emph{stochastic} risk. 

Instead of training separately in this manner, meta-learning presumes an {\em a priori} relation between the tasks in \( \mathcal{T}_k \) and exploits this fact. In particular, MAML seeks a ``{\em launch model}'' such that when faced with data arising from a new task, the agent would be able to update the ``launch model'' with a small number of task-specific gradient updates. It is common to allow for multiple gradient steps for task adaptation. For the analytical part of this work, we will restrict ourselves to a single gradient step for simplicity. Nevertheless, our experimental results suggest that the theoretical conclusions hold more broadly even when allowing for multiple gradient updates to the launch model. With a single gradient step, agent \( k \) can seek a launch model by minimizing the modified risk function:
\begin{equation}\label{real risk}
    \min_{w \in \mathds{R}^M} \jkb (w)\triangleq \etk\rjkt\Big(w-\al\rgjkt(w)\Big)
\end{equation}
The resulting gradient vector is given by (assuming the possibility of exchanging expectations and gradient operations, which is valid under mild technical conditions):
\begin{align}\label{real risk grad}
    &\gjkb (w) \triangleq \\ &\etk\Big[\Big(I-\al \rhjkt(w)\Big)  \rgjkt\Big(w-\al \rgjkt(w)\Big)\Big] \notag
\end{align}
where \(\al>0\) is the step size parameter. In practice, due to lack of information about \( \pi_{k} \) and the distribution of \(  \boldsymbol{x}_{k}^{(t)}\), evaluation of \eqref{real risk} and \eqref{real risk grad} is not feasible. It is common to collect realizations  of data and replace \eqref{real risk grad} by a stochastic gradient approximation:
\begin{align}\label{maml grad}
    \nabla \overline{Q_{k}}(w)   &\triangleq  \averagetasks \Big[ \Big(I-\al \rhqkt(w;\rxin)\Big)  \notag \\ &\rgqkt\Big(w-\al \rgqkt(w;\rxin)\ ;\ \rxout\Big) \Big ]  
\end{align}
where \( \rxin \), \(\rxout\) are two random batches of data\footnote{Different batches of data are used while computing the inner and outer gradients. The reason is that we want our model to adapt to models that perform well on data that is not used for training. If the two batches were the same, then this would train the launch model to be an initialization for task-specific models that memorize their training data. This memorization would get in the way of generalization.}, \( \boldsymbol{\mathcal{S}}_{k} \subset \mathcal{T}_{k}\) is a random batch of tasks, and \(\abs{\mathcal{S}_{k}} \) is the number of selected tasks. We assume that all elements of \( \rxin \), \(\rxout\) are independently sampled from the distribution of \( \boldsymbol{x}_k^{(t)} \) and all tasks \( \boldsymbol{t} \in  \boldsymbol{\mathcal{S}}_k \) are independently sampled from \( \mathcal{T}_k \).

In a \emph{non-cooperative} MAML setting, each agent \( k \) would optimize \eqref{real risk} in an effort to obtain a launch model that is likely to adapt quickly to tasks similar to those encountered in \( \mathcal{T}_k \). In a \emph{cooperative} multi-agent setting, however, one would expect transfer learning to occur between agents. This motivates us to seek a decentralized scheme where the launch model obtained by agent \( k \) is likely to generalize well to tasks similar to those observed by agent \( \ell \) during training, for any pair of agents \( k, \ell \). This can be achieved by pursuing a launch model that optimizes instead the aggregate risk:
\begin{equation}\label{aggregate cost}
    \min_{w \in \mathds{R}^M} \jb (w) \triangleq \frac{1}{K} \sum_{k=1}^{K} \jkb (w)
\end{equation}
By pursuing this network objective in place of the individual objectives, the effective number of tasks and data each agent is trained on is increased and hence a better generalization performance is expected. Even though both the centralized and decentralized strategies seek a solution to \eqref{aggregate cost}, in the decentralized strategy, the agents rely only on their immediate neighbors and there is no central processor. 
\subsection{Related Work}
Early works on meta-learning or learning to learn date back to \citep{sch87,sch92,Bengio91,bengio92}. Recently, there has been increased interest in meta-learning with various approaches such as learning an optimization rule \citep{Andrychowicz16,Ravi17} or learning a metric that compares support and query samples for few-shot classification \citep{koch2015siamese,VinyalsMatchingNetworks}.

In this paper, we consider a parameter-initialization-based meta-learning algorithm. This kind of approach was introduced by MAML \citep{FinnAL17}, which aims to find a good initialization (launch model) that can be adapted to new tasks rapidly. It is model-agnostic, which means it can be applied to any model that is trained with gradient descent. MAML has shown competitive performance on benchmark few-shot learning tasks. Many algorithmic extensions have also been proposed by \citep{reptile,ProbabilisticMAML,Raghu2020Rapid,metasgd} and several works have focused on the theoretical analysis and convergence of MAML \citep{FallahConvergence,Ji2020MultiStepMM,Nonconvex_online_ml,provablebalcan} in single-agent settings.

A different line of work \citep{BalcanAdaptiveGradient,ImprovingFL,Fallah2020Personalized,FederatedML} studies meta-learning in a federated setting where the agents communicate with a central processor in a manner that keeps the privacy of their data. In particular, \citep{Fallah2020Personalized} and \citep{FederatedML} propose algorithms that learn a global shared launch model, which can be updated by a few agent-specific gradients for personalized learning. In contrast, we consider a decentralized scheme where there is no central node and only \emph{localized} communications with neighbors occur. This leads to a more scalable and flexible system and avoids communication bottleneck at the central processor.

Our extension of MAML is based on the diffusion algorithm for decentralized optimization \citep{Sayed14}. While there exist many useful decentralized optimization strategies such as consensus \citep{Nedic09,Xiao03,Yuan16} and diffusion \citep{Sayed14,Sayed14proc}, the latter class of protocols has been shown to be particularly suitable for adaptive scenarios where the solutions need to adapt to drifts in the data and models. Diffusion strategies have also been shown to lead to wider stability ranges and lower mean-square-error performance than other techniques in the context of adaptation and learning due to an inherent symmetry in their structure. Several works analyzed the performance of diffusion strategies such as \citep{Sayed14proc,Nassif16,Chen15transient,Chen15performance}. The works \citep{Vlaski19,Vlaski19nonconvexP1} examined diffusion under non-convex losses and stochastic gradient conditions, which are applicable to our work with proper adjustment since the risk function for MAML includes a gradient term as an argument for the risk function.

\section{Dif-MAML}
Our algorithm is based on the Adapt-then-Combine variant of the diffusion strategy \citep{Sayed14}.

\subsection{Diffusion (Adapt-then-Combine)}
The diffusion strategy is applicable to scenarios where \( K \) agents, connected via a graph topology \(A = [a_{\ell k}]\), collectively try to minimize an aggregate risk \( {\jb (w) \triangleq \frac{1}{K} \sum_{k=1}^{K} \jkb (w) }\) , which includes the setting \eqref{aggregate cost} considered in this work. To solve this objective, at every iteration \(i\), each agent \( k \) simultaneously performs the following steps:
\begin{subequations}
\begin{align}
  \boldsymbol{\phi}_{k,i} &= \w_{k,i-1} - \mu \gqkb(\w_{k,i-1})\label{eq:adapt}\\
  \w_{k,i} &= \sum_{\ell=1}^{K} a_{\ell k} \boldsymbol{\phi}_{\ell,i}\label{eq:combine}
\end{align}
\end{subequations}
The coefficients $\{a_{\ell k}\}$ are non-negative and add up to one: \begin{equation}
{\sum_{\ell=1}^{K} a_{\ell k}=1},\;
{a_{\ell k}>0}\;\;\mbox{\rm if agents $\ell$ and $k$ are connected} \notag
\end{equation}
For example, matrix $A$ can be selected as the Metropolis rule.

Expression \eqref{eq:adapt} is an \emph{adaptation} step where all agents simultaneously obtain intermediate states \( \boldsymbol{\phi}_{k,i}\) by a stochastic gradient update. Recall that \( \gqkb(\w_{k,i-1}) \) from \eqref{maml grad} is the stochastic approximation of the exact gradient \(\gjkb (\w_{k,i-1}) \) from \eqref{real risk grad} . Expression \eqref{eq:combine} is a \emph{combination} step where the agents combine their neighbors' intermediate steps to obtain updated iterates \( \w_{k,i} \).
\subsection{Diffusion-based MAML (Dif-MAML)}
We present the proposed algorithm for decentralized meta-learning in Algorithm 1. Each agent is assigned an initial launch model. At every iteration, the agents sample a batch of i.i.d. tasks from their agent-specific distribution of tasks. Then, in the inner loop, task-specific models are found by applying task-specific stochastic gradients to the launch models. Subsequently, in the outer loop, each agent computes an intermediate state for its launch model based on an update consisting of the sampled batch of tasks. A standard MAML algorithm would assign the intermediate states as the revised launch models and stop there, without any cooperation among the agents. However, in Dif-MAML, the agents cooperate and update their launch models by combining their intermediate states with the intermediate states of their neighbors. This helps in the transfer of knowledge among agents.
\begin{algorithm}[tbh]
Initialize the \emph{launch} models \( \{w_{k,0}\}_{k=1}^{K} \) \\
\While {not done}{
\For{all agents}{
Agent \( k \) samples a batch of \( i.i.d. \) tasks \( \boldsymbol{\mathcal{S}}_{k,i} \) from \( \mathcal{T}_k \) \\
\For{all tasks \( \boldsymbol{t} \in \boldsymbol{\mathcal{S}}_{k,i} \)}{
Evaluate \( \rgqkt\Big(\w_{k,i-1};\rds_{in,i}^{(t)}\Big) \) using a batch of \( i.i.d. \) data \( \rds_{in,i}^{(t)} \)\\
Set \emph{task-specific} models \( \w_{k,i}^{(t)} = \w_{k,i-1}-\al \rgqkt\Big(\w_{k,i-1};\rds_{in,i}^{(t)}\Big) \)
}
Compute \emph{intermediate states} \( \boldsymbol{\phi}_{k,i} = \w_{k,i-1}-(\mu/\abs{\mathcal{S}_{k,i}}) \sum_{\boldsymbol{t} \in \boldsymbol{\mathcal{S}}_{k,i}} \rgqkt \Big(\w_{k,i}^{(t)};\rds_{o,i}^{(t)}\Big) \) using a batch of \( i.i.d. \) data \( \rds_{o,i}^{(t)} \) for each task (Check \eqref{maml grad} to see gradient expression explicitly)\\
}
\For{all agents}{
Update the launch models by combining the intermediate states \( \w_{k,i} = \sum_{\ell=1}^{K} a_{\ell k} \boldsymbol{\phi}_{\ell,i} \)
}
\( i \leftarrow i+1 \)
}
 \caption{Dif-MAML}
\end{algorithm}
\section{Theoretical Results}
In this section, we provide convergence analysis for Dif-MAML in non-convex environments.
\subsection{Assumptions}

\begin{assumption}[\textbf{Lipschitz gradients}]\label{assumption lipshtiz gradient} For each agent \( k \) and task \( t \in \mathcal{T}_k \), the gradient \( \gqkt (\cdot;\cdot) \)  is Lipschitz, namely, for any \( w , u \in \mathds{R}^{M} \) and  \( \x_{k}^{(t)} \) denoting a data point:
\begin{equation}\label{lipschitz grad data specific}
\left \|\gqkt\Big(w;\boldsymbol{x}_{k}^{(t)}\Big)-\gqkt\Big(u;\boldsymbol{x}_{k}^{(t)}\Big)\right \| \leq L^{\boldsymbol{x}_{k}^{(t)}}\norm{w-u}
\end{equation}{}

We assume the second-order moment of the Lipschitz constant is bounded by a data-independent constant:
\begin{equation}\label{lipschitz grad expectation data specific}
    \e_{x_{k}^{(t)}}\Big(L^{\boldsymbol{x}_{k}^{(t)}}\Big)^{2}\leq \Big(L_k^{(t)}\Big)^{2}
\end{equation}{}
\end{assumption}

We establish in Appendix \ref{appendix lip grad proofs} that Assumption \ref{assumption lipshtiz gradient} holds for gradients under a batch of data. In this paper, for simplicity, we will mostly work with \( L \triangleq \max\limits_{k}\max\limits_{t}  L_k^{(t)} \).

\begin{assumption} [\textbf{Lipschitz Hessians}]\label{assumption lipschitz hessians} For each agent \( k \) and task \( t \in \mathcal{T}_k \), the Hessian \( \hqkt (\cdot;\cdot) \)  is Lipschitz in expectation, namely, for any \( w , u \in \mathds{R}^{M} \) and  \( \x_{k}^{(t)} \) denoting a data point:
\begin{equation}\label{lipschitz hessian data specific}
    \e_{x_{k}^{(t)}}\left \| \hqkt\Big(w;\x_{k}^{(t)}\Big)-\hqkt\Big(u;\x_{k}^{(t)}\Big)\right \| \leq \rho_k^{(t)}\norm{w-u}
\end{equation}{}
\end{assumption}

We establish in Appendix \ref{appendix lip hessian proofs} that Assumption \ref{assumption lipschitz hessians} holds for Hessians under a batch of data. In this paper, for simplicity, we will mostly work with \( \rho \triangleq \max\limits_{k}\max\limits_{t}  \rho_k^{(t)}\).
\begin{assumption}[\textbf{Bounded gradients}]\label{assumption bounded gradients} For each agent \( k \) and task \( t \in \mathcal{T}_k \), the gradient \( \gqkt (\cdot;\cdot) \)  is bounded in expectation, namely, for any \( w \in \mathds{R}^{M} \) and  \( \x_{k}^{(t)} \) denoting a data point:
\begin{equation}\label{bounded grad data specific}
    \e_{x_{k}^{(t)}}\left \| \gqkt\Big(w;\x_{k}^{(t)}\Big)\right \| \leq B_{k}^{(t)}
\end{equation}{}
\end{assumption}

We establish in Appendix \ref{appendix bounded proofs} that Assumption \ref{assumption bounded gradients} holds for gradients under a batch of data. In this paper, for simplicity, we will mostly work with \( B \triangleq \max\limits_{k}\max\limits_{t}  B_k^{(t)}\).

\begin{assumption}[\textbf{Bounded noise moments}]\label{assumption fourth order data expectation} For each agent \( k \) and task \( t \in \mathcal{T}_k \), the gradient \( \gqkt (\cdot;\cdot) \) and the Hessian \( \hqkt (\cdot;\cdot) \) have bounded fourth-order central moments, namely, for any \( w \in \mathds{R}^{M} \):
\begin{align}
    &\e_{x_{k}^{(t)}} {\left \| \gqkt\Big(w;\x_{k}^{(t)}\Big)-\gjkt(w) \right \| }^{4}\leq \sigg^{4} \label{sigma g data specific}\\ 
    &\e_{x_{k}^{(t)}}{\left \| \hqkt\Big(w;\x_{k}^{(t)}\Big)-\hjkt(w)\right \| }^{4}\leq \sigh^{4} \label{sigma h data specific}
\end{align}{}
\end{assumption}{}

We establish in Appendix \ref{appendix sigma g h proofs} that Assumption \ref{assumption fourth order data expectation} holds for gradients and Hessians under a batch of data.

Defining the mean of the risk functions of the tasks in \( \mathcal{T}_k \) as \( \jk (w) \triangleq \etk \rjkt (w) \), we have the following assumption on the relations between the tasks of a particular agent.

\begin{assumption} [\textbf{Bounded task variability}]\label{assumption fourth order task expectation}For each agent \( k \) , the gradient \( \rgjkt (\cdot) \) and the Hessian \( \rhjkt (\cdot) \) have bounded fourth-order central moments, namely, for any \( w \in \mathds{R}^{M} \):
\begin{align}
    &\etk{\left \| \rgjkt(w)-\gjk(w) \right \| }^{4}\leq \gamg^{4} \\
    &\etk{\left \| \rhjkt(w)-\hjk(w) \right \|}^{4}\leq \gamh^{4}
\end{align}{}
\end{assumption}
Note that we do not assume any constraint on the relations between tasks of different agents.
\begin{assumption}[\textbf{Doubly-stochastic combination matrix}]\label{assumption doubly stochastic matrix}
The weighted combination matrix \( A=[a_{\ell k}] \) representing the graph is doubly-stochastic. This means that the matrix has non-negative elements and satisfies:
\begin{equation}\label{eq:doubly-stochastic matrix properties}
  A\mathds{1}=\mathds{1}, \quad A^{\T}\mathds{1}=\mathds{1}
\end{equation}
The matrix $A$ is also primitive which means that a path with positive weights can be found between any arbitrary nodes $(k,\ell)$, and moreover at least one $a_{kk}>0$ for some $k$. 
\end{assumption}

\newcolumntype{C}{ >{\centering\arraybackslash} m{1.75cm} }
\newcolumntype{D}{ >{\centering\arraybackslash} m{1.5cm} }

\subsection{Alternative MAML Objective}

The stochastic MAML gradient \eqref{maml grad}, because of the gradient within a gradient, is not an \emph{unbiased} estimator of \eqref{real risk grad}. We consider the following alternative objective in place of \eqref{real risk}:
\begin{align}\label{adjusted objective formula}
    \jkbh(w) \triangleq \etk\exin\rjkt\Big(w-\al\rgqkt(w;\rxin)\Big)
\end{align}
The gradient corresponding to this objective is the expectation of the stochastic MAML gradient \eqref{maml grad}:
\begin{align}\label{adjusted objective grad formula}
    \gjkbh(w)  = \e\gqkb(w)
\end{align}
See Table \ref{tab:notations} for a summary of the notation in the paper. We establish \eqref{adjusted objective grad formula} in Appendix \ref{appendix adj grad proof}. This means that the stochastic MAML gradient \eqref{maml grad} is an unbiased estimator for the gradient of the alternative objective \eqref{adjusted objective formula}.

While the MAML objective \eqref{real risk} captures the goal of coming up with a launch model that performs well after a \emph{gradient step}, the \emph{adjusted objective} \eqref{adjusted objective formula} searches for a launch model that performs well after a \emph{stochastic gradient step}.  Using the adjusted objective allows us to analyze the convergence of Dif-MAML by exploiting the fact that it results in an unbiased stochastic gradient approximation. 

In the following two lemmas, we will perform perturbation analyses on the MAML objective \( \jkb(w) \) and the adjusted objective \( \jkbh(w) \). We will work with \( \jkbh (w) \) afterwards. At the end of our theoretical analysis, we will use the perturbation results to establish convergence to stationary points for both objectives.

\begin{table}
\ra{1.3}
\begin{tabular}{ C D D D }\toprule
 & Single-task&Meta-objective&Meta-gradient\\
 \cline{2-4}
 Risk function   & \(\jkt \)    & \( \jkb \)&  \( \gjkb \) \\
 
 Adjusted Risk   & X    & \( \jkbh \)&   \( \gjkbh \)\\
 
  Stochastic Approx.   & \( \qkt \)   & X &   \( \gqkb \)\\
  \bottomrule
\end{tabular}
\caption{Summary of some notation used in the paper.}\label{tab:notations}
\end{table}

\begin{lemma}[\textbf{Objective perturbation bound}]\label{perturbation analysis formula} Under assumptions \ref{assumption lipshtiz gradient},\ref{assumption bounded gradients},\ref{assumption fourth order data expectation}, for each agent \( k \), the disagreement between \( \jkb ( \cdot ) \) and \( \jkbh ( \cdot ) \) is bounded, namely, for any \( w \in \mathds{R}^{M} \): \footnote{In this paper, for simplicity, we assume that for each agent \( k \) and task \( t \in \mathcal{T}_{k} \), \( \abs{\rxin} = \xinsize \) and \( \abs{\rxout} = \xosize \).}
\begin{equation}
        {\left | \jkb(w)-\jkbh(w) \right |}\leq\frac{\al^{2}L\sigg^{2}}{2\xinsize}+\frac{B\al\sigg}{\sqrt{\xinsize}}
\end{equation}{}
\begin{proof}{}
See Appendix \ref{appendix perturb analysis}.
\end{proof}
\end{lemma}
Next, we perform a perturbation analysis at the gradient level.
\begin{lemma}[\textbf{Gradient perturbation bound}]\label{perturbation analysis gradient formula} Under assumptions \ref{assumption lipshtiz gradient},\ref{assumption bounded gradients},\ref{assumption fourth order data expectation}, for each agent \( k \), the disagreement between \( \gjkb ( \cdot ) \) and \( \gjkbh ( \cdot ) \) is bounded, namely, for any \( w \in \mathds{R}^{M} \):
\begin{equation}\label{eq:pert_analysis_grad}
        {\left \| \gjkb(w)-\gjkbh(w) \right \|}\leq(1+\al L)\frac{\al L \sigg}{\sqrt{\xinsize}}+\frac{B\al\sigh}{\sqrt{\xinsize}}
\end{equation}{}
\begin{proof}{}
See Appendix \ref{appendix perturb grad analysis}.
\end{proof}
\end{lemma}
Lemmas \ref{perturbation analysis formula} and \ref{perturbation analysis gradient formula} suggest that the standard MAML objective and the adjusted objective get closer to each other with decreasing inner learning rate \( \alpha \) and increasing inner batch size \( \xinsize\). Next, we establish some properties of the adjusted objective, which will be called upon in the analysis.
\begin{lemma}[\textbf{Bounded gradient of adjusted objective}]\label{adj objective bounded grad formula}Under assumptions \ref{assumption lipshtiz gradient},\ref{assumption bounded gradients}, for each agent \( k \), the gradient \( \gjkbh (\cdot) \) of the adjusted objective is bounded, namely, for any \( w \in \mathds{R}^{M} \):
\begin{equation}
    {\left \| \gjkbh(w) \right \|}\leq \bbh
\end{equation}
where \( \bbh \triangleq (1+\al L)B \) is a non-negative constant.
\begin{proof}
See Appendix \ref{appendix bounded gradient adj objective proof}.
\end{proof}{}
\end{lemma}{}

\begin{lemma}[\textbf{Lipschitz gradient of adjusted objective}]\label{adj objective lipschitz formula} Under assumptions \ref{assumption lipshtiz gradient}-\ref{assumption bounded gradients}, for each agent \( k \), the gradient \( \gjkbh (\cdot) \) of adjusted objective is Lipschitz, namely, for any \( w,u \in \mathds{R}^{M} \):
\begin{equation}
    {\left \| \gjkbh(w)-\gjkbh(u) \right \|}\leq \lbh \norm{w-u}
\end{equation}
where \( \lbh \triangleq (L(1+\al L)^{2}+\al \rho B) \) is a non-negative constant.
\begin{proof}
See Appendix \ref{appendix adj obj lipschitz}.
\end{proof}
\end{lemma}

\begin{lemma}[\textbf{Gradient noise for adjusted objective}]\label{adj objective noise bound} Under assumptions \ref{assumption lipshtiz gradient}-\ref{assumption fourth order task expectation}, the gradient noise defined as \( \gqkb (w)-\gjkbh (w) \) is bounded for any \( w \in \mathds{R}^{M} \), namely:
\begin{equation}
    \e{\left \| \gqkb (w)-\gjkbh (w) \right \|}^{2}\leq C^{2}
\end{equation}
for a non-negative constant \( C^{2} \), whose expression is given in \eqref{eq:appendix c expression} in Appendix \ref{appendix adj obj noise process}.
\begin{proof}
See Appendix \ref{appendix adj obj noise process}.
\end{proof}
\end{lemma}

\subsection{Evolution Analysis}

In this section, we analyze the Dif-MAML algorithm over the network. The analysis is similar to \citep{Vlaski19,Vlaski19nonconvexP1}. We first prove that agents cluster around the network centroid in \( O(\log \mu) = o(1/\mu)\) iterations, then show that this centroid reaches an \( O(\mu)\)-mean-square-stationary point in at most \( O(1/\mu^2)\) iterations. Figure \ref{FIG:summary of analysis} summarizes the analysis.

 \begin{figure*}[h!]
 \centering
  \includegraphics[width=0.75\textwidth]{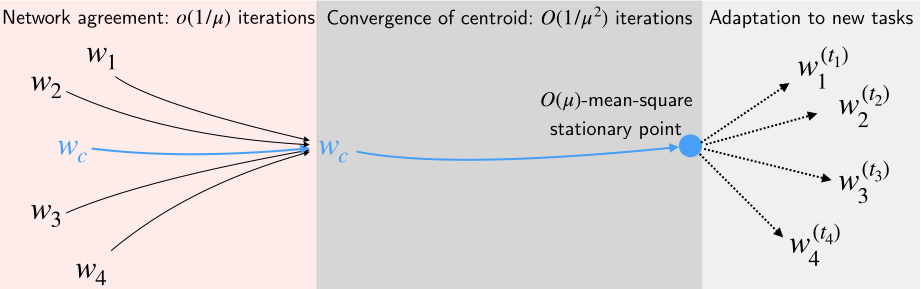}
 \caption{Diagram of the analysis. Agents cluster around a common network centroid, and this centroid reaches a stationary point of the MAML objective during meta-training. Subsequently, agents can use this launch model in order to adapt to new tasks.}\label{FIG:summary of analysis}%
 \end{figure*}
 
The network centroid is defined as \( {\w_{c, i} \triangleq \frac{1}{K} \sum_{k=1}^K \w_{k, i}} \). It is an average of the agents' parameters. In the following theorem, we study the difference between the centroid launch model and the launch model for each agent \( k \).

\begin{theorem}[\textbf{Network disagreement}]\label{lemma network disagreement} Under assumptions \ref{assumption lipshtiz gradient}-\ref{assumption doubly stochastic matrix}, network disagreement between the centroid launch model and the launch models of each agent \( k \) is bounded after \(O(\log \mu)=o(1/\mu) \) iterations, namely:
\begin{align}
   \frac{1}{K} \sum_{k=1}^K \e {\left \| \w_{k,i} - \w_{c, i} \right \|}^2 \le& \mu^2 \frac{\lambda_2^2}{{(1-\lambda_2)}^2} \left( \bbh^2 + C^2 \right) \notag \\ &+ O(\mu^3)
\end{align}
for
\begin{equation}
    i \ge \frac{3\log \mu }{\log \lambda_2}+O(1)=o(1/\mu)
\end{equation}
where \( \lambda_2\) is the mixing rate of the combination matrix \( A \), i.e., it is the spectral radius of \( A^{\T} - \frac{1}{K} \mathds{1}_K \mathds{1}_K^{\T}\).
\end{theorem}
\begin{proof}
See Appendix \ref{appendix proof network disagreement lemma}.
\end{proof}

In Theorem \ref{lemma network disagreement}, we proved that the disagreement between the centroid launch model and agent-specific launch models is bounded after sufficient number of iterations. Therefore, we can use the centroid model as a deputy for all models and establish the convergence properties on that.

\begin{theorem}[\textbf{Stationary points of adjusted objective}]\label{COR:STATIONARY_POINTS}
	In addition to assumptions \ref{assumption lipshtiz gradient}-\ref{assumption doubly stochastic matrix}, assume that \( \jbh(w) \) is bounded from below, i.e., \( \jbh(w) \ge \jbh^o \). Then, the centroid launch model \( \w_{c, i} \) will reach an \( O(\mu) \)-mean-square-stationary point in at most \( O\left(1/\mu^2\right) \) iterations. In particular, there exists a time instant \( i^{\star} \) such that:
	\begin{equation}\label{eq:stationary_points}
		\e {\left \|\gjbh(\w_{c, i^{\star}})\right\|}^2 \le 2 \mu \lbh C^2 + O(\mu^2)
	\end{equation}
	and
	\begin{equation}\label{eq:istar number of iterations}
		i^{\star} \le \left(\frac{2(\jbh(w_0) - \jbh^o)}{\lbh C^2}\right) 1/\mu^2 + O(1/\mu)
	\end{equation}
\end{theorem}
\begin{proof}
See Appendix \ref{appendix stationary points corollary proof}.
\end{proof}
Next, we prove that the same analysis holds for the standard MAML objective, using the gradient perturbation bound for the adjusted objective (Lemma \ref{perturbation analysis gradient formula}).

\begin{corollary}[\textbf{Stationary points of MAML objective}]\label{COR:STATIONARY_POINTS_MAML_OBJ)}
Assume that the same conditions of Theorem \ref{COR:STATIONARY_POINTS} hold. Then, the centroid launch model \( \w_{c, i} \) will reach an \( O(\mu) \)-mean-square-stationary point, up to a constant, in at most \( O\left(1/\mu^2\right) \) iterations. Namely, for time instant \( i^{\star} \) defined in \eqref{eq:istar number of iterations}:
\begin{align}\label{eq:stationary_points_maml_obj}
		\e {\left \|\gjb (\w_{c, i^{\star}})\right\|}^2 \le &4 \mu \lbh C^2 + O(\mu^2) \\ &+ 2 \left ( (1+\al L)\frac{\al L \sigg}{\sqrt{\xinsize}}+\frac{B\al\sigh}{\sqrt{\xinsize}} \right )^2 \notag
	\end{align}
\end{corollary}
\begin{proof}
See Appendix \ref{appendix maml stationary points corollary proof}.
\end{proof}
Corollary \ref{COR:STATIONARY_POINTS_MAML_OBJ)} states that the centroid launch model can reach an \( O(\mu) \)-mean-square-stationary point for sufficiently small inner learning rate \( \alpha \) and for sufficiently large inner batch size \( \xinsize \), in at most \( O\left(1/\mu^2\right) \) iterations. Note that as \( \mu \rightarrow 0 \), the number of iterations required for network agreement (\( O(\log \mu) = o (1/\mu) \)) becomes negligible compared to the number of iterations necessary for convergence \( (O(1/\mu^2)) \).

\begin{figure*}[!t]
\centering
\subfloat[a][]{
  \includegraphics[width=.3\linewidth]{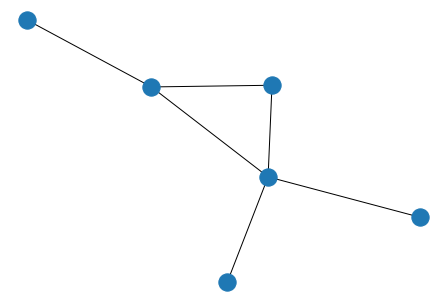}\label{fig:original network}%
}
\hfil
\subfloat[b][]{
  \includegraphics[width=.3\linewidth]{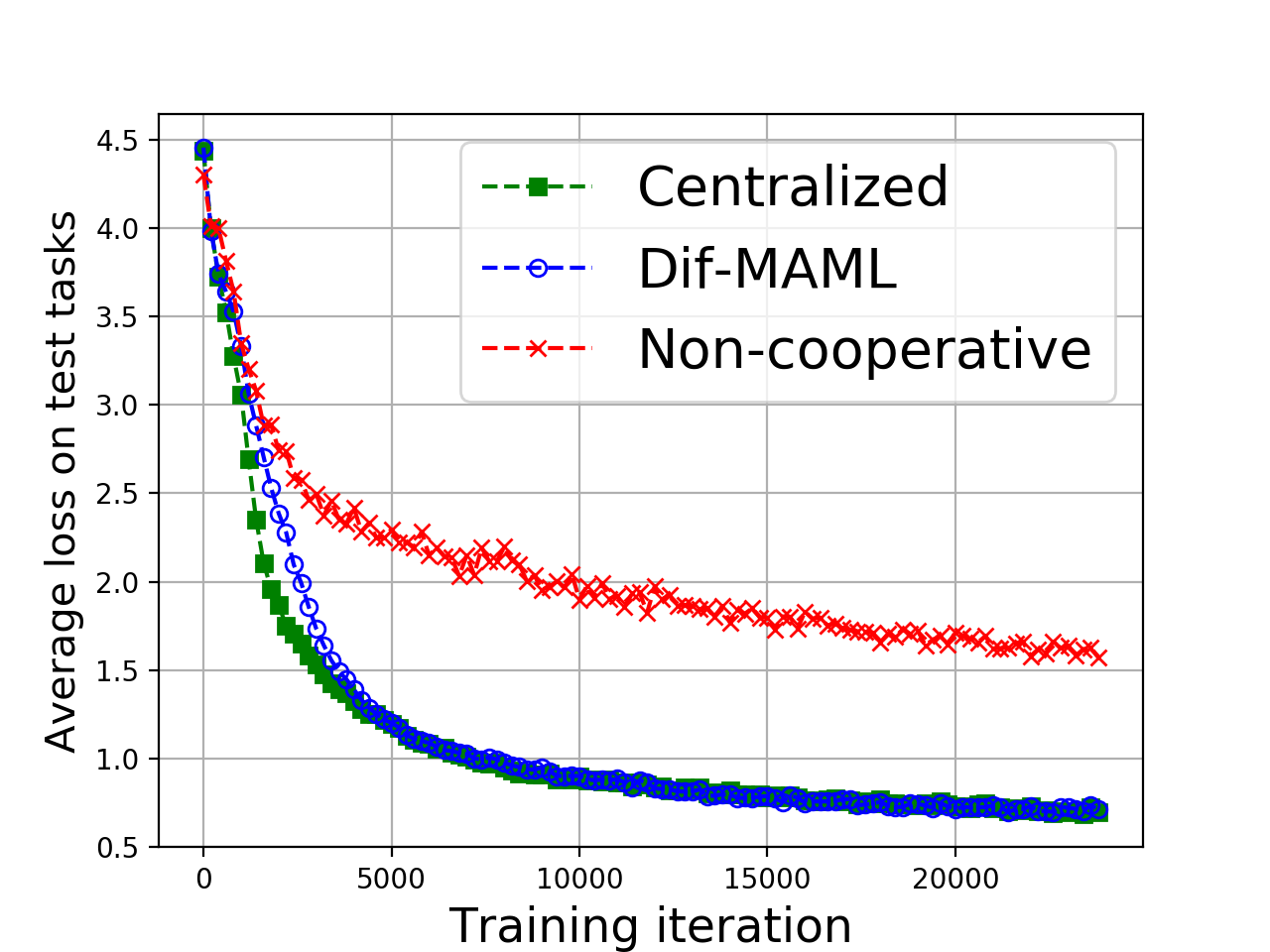}\label{fig:reg_training}%
}
\hfil
\subfloat[c][]{
  \includegraphics[width=.3\linewidth]{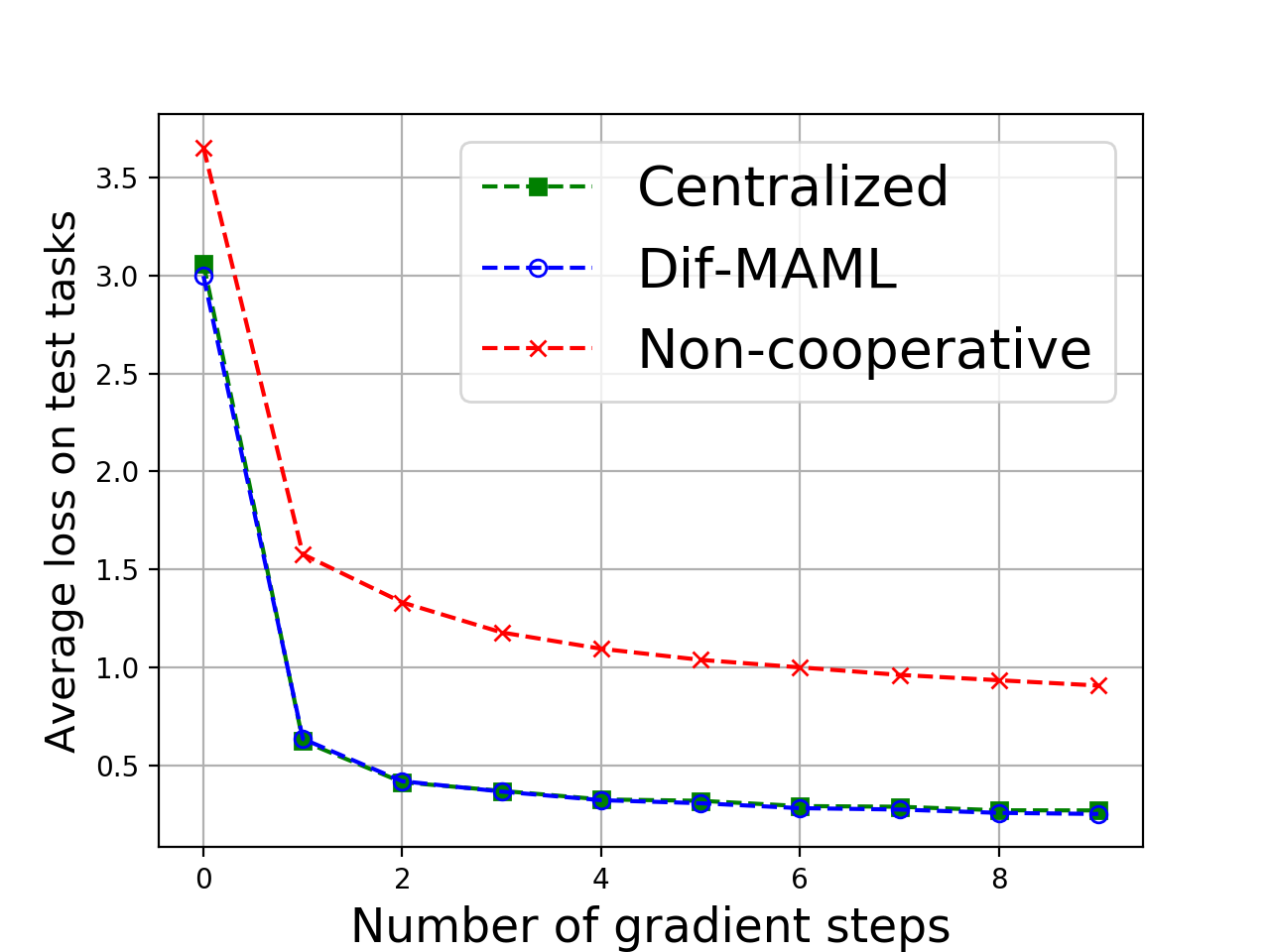}\label{fig:reg_test}
}
\caption{(a) Graph underlying the network (b) Regression-test losses during training - Adam (c) Regression-test losses with respect to number of gradient steps after training - Adam}\label{fig:1}
\end{figure*}

\begin{figure}[h]
\centering
\includegraphics[width=0.8\linewidth]{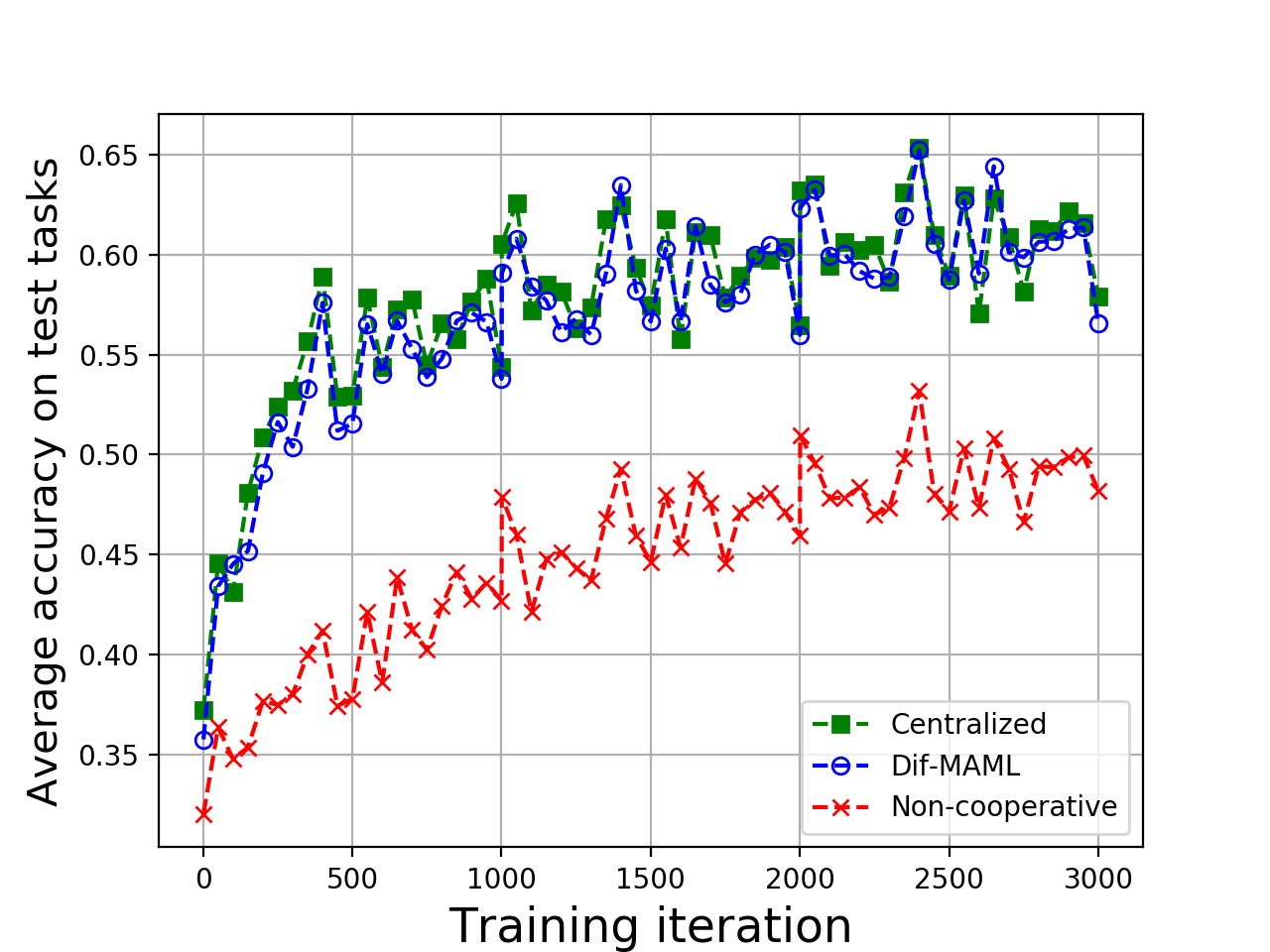}
\caption{MiniImagenet test accuracies during training process: 5-way 5-shot Adam}
\end{figure}\label{fig:miniimagenet 5w5s}

\section{Experiments}
In this section, we provide experimental evaluations. In particular, we present comparisons between the centralized, diffusion-based decentralized, and non-cooperative strategies. Our demonstrations cover both regression and classification tasks. Even though our theoretical analysis is general with respect to various learning models, for the experiments, our focus is on neural networks. 

For all cases, we consider the network with the underlying graph in Figure \ref{fig:original network} with \( K=6 \) agents. The centralized strategy corresponds to a central processor that has access to all data and tasks. Note that this is equivalent to having a network with a fully-connected graph. The non-cooperative strategy represents a solution where agents do not communicate with each other. In other words, they all try to learn separate launch models. 

\subsection{Regression}

For regression, we consider the benchmark from \citep{FinnAL17}. In this setting, each task requires predicting the output of a sine wave from its input. Different tasks have different amplitudes and phases. Specifically, the phases are varied between \( [0,\pi] \) for each agent. However, the agents have access to different task distributions since the amplitude interval \( [0.1,5.0] \) is evenly partitioned into \( K=6 \) different intervals and each agent is equipped with one of them. The outer-loop optimization is based on Adam \citep{Adam}. See Appendix \ref{appendix regression details} for additional details and Appendix \ref{appendix additional plots} for results of an experiment with SGD. 

Every 200th iteration, the agents are tested over 1000 tasks. All agents are evaluated with the same tasks, which stem from the intervals \( [0.1,5.0] \) for amplitude and \( [0,\pi] \) for phase. The results are shown in Figure \ref{fig:reg_training}. It can be seen that Dif-MAML quickly converges to the centralized solution and clearly outperforms the non-cooperative solution throughout the training. This suggests that cooperation helps even when agents have access to different task distributions. Moreover, we also test the performance after training with respect to number of gradient updates for adaptation in Figure \ref{fig:reg_test}. It is visible that the match between the centralized and decentralized solutions does not change and the performance of the non-cooperative solution is still inferior. Note that this plot is also showing the average performance on 1000 tasks.

\subsection{Classification}

For classification, we consider widely used few-shot image recognition tasks on the Omniglot \citep{Omniglot} and MiniImagenet \citep{Ravi17} datasets (see Appendix \ref{appendix dataset details} for dataset details). In contrast to the regression experiment, in these simulations, all agents have access to the same tasks and data. However, in the centralized and decentralized strategies, the effective number of samples is larger as we limit the number of data and tasks processed in one agent. See Appendix \ref{appendix classification details} for details on the architecture and hyperparameters. Average accuracy on test tasks at every 50th training iteration is shown in Fig. \ref{fig:miniimagenet 5w5s} for MiniImageNet 5-way 5-shot setting trained with Adam. See Appendix \ref{appendix additional plots} for additional experiments on 5-way 1-shot MiniImagenet, 5-way 1-shot and 20-way 1-shot Omniglot and SGD variants. Similar to the regression experiment, the decentralized solution matches the centralized solution and is substantially better than the non-cooperative solution.

Moreover, we observe that batch normalization \citep{Batch_Normalization} is necessary for applying Dif-MAML, and diffusion in general, on neural networks since the combination step \eqref{eq:combine} reduces the variance of the weights due to averaging.

\section{Conclusion}

In this paper, we proposed a decentralized algorithm for meta-learning. Our theoretical analysis establishes that the agents' launch models cluster quickly in a small region around the centroid model and this centroid model reaches a stationary point after sufficient iterations. We illustrated by means of extensive experiments on regression and classification problems that the performance of Dif-MAML indeed consistently coincides with the centralized strategy and surpasses the non-cooperative strategy significantly.

\acknowledgments

The authors would like to thank Yigit Efe Erginbas for helpful discussions on the experiments.

\bibliography{multi_maml}

%%% APPENDIX %%%
% If the appendix exceeds the page limit, it must be submitted as a supplementary material.

\clearpage
\appendix

% To improve readability, you may use a single-column format for the supplementary material, which can be achieved by the \onecolumn command.
\onecolumn

\newpage
\begin{center}
 \Large\textbf{APPENDIX}
\end{center}
 
\section{The Implications of Assumptions for Batches of Data}

In Appendices \ref{appendix lip grad proofs}-\ref{appendix sigma g h proofs} we denote a batch of data by \(  \boldsymbol{\mathcal{X}}_{k}^{(t)} \), its size by \( N_k^{(t)} \), and its elements by \( \left\{ \boldsymbol{x}_{k, n}^{(t)} \right\}_{n=1}^{N_k^{(t)}} \)

\subsection{The Implication of Assumption \ref{assumption lipshtiz gradient}}\label{appendix lip grad proofs}
Assumption \ref{assumption lipshtiz gradient} implies for the stochastic gradient constructed using a batch:
\begin{equation}\label{lipschitz grad dataset specific}
\left \| \gqkt\Big(w;\boldsymbol{\mathcal{X}}_{k}^{(t)}\Big)-\gqkt\Big(u;\boldsymbol{\mathcal{X}}_{k}^{(t)}\Big)\right \|\leq L^{\boldsymbol{\mathcal{X}}_{k}^{(t)}} \norm{w-u}
\end{equation}
where \(L^{\boldsymbol{\mathcal{X}}_{k}^{(t)}} \triangleq \frac{1}{N_{k}^{(t)}}\sum_{n=1}^{N_{k}^{(t)}}L^{\boldsymbol{x}_{k,n}^{(t)}} \). Moreover,
\begin{equation}\label{lipschitz grad expectation}
    \e_{\mathcal{X}_{k}^{(t)}}\Big(L^{\boldsymbol{\mathcal{X}}_{k}^{(t)}}\Big)^{2} \leq \Big(L_k^{(t)}\Big)^{2}
\end{equation}\qed
\begin{proof}
For the stochastic gradients under a batch of data:
\begin{align}
    {\left \| \gqkt(w;\boldsymbol{\mathcal{X}}_{k}^{(t)})-\gqkt(u;\boldsymbol{\mathcal{X}}_{k}^{(t)}) \right \|}&=\left \| \frac{1}{N_{k}^{(t)}} \sum_{n=1}^{N_{k}^{(t)}}\Big(\rgqkt(w;\x_{k,n}^{(t)})-\gqkt(u;\x_{k,n}^{(t)})\Big) \right \| \notag \\
    &\stackrel{(a)}{\leq}\frac{1}{N_{k}^{(t)}}\sum_{n=1}^{N_{k}^{(t)}}{\left \| \rgqkt(w;\x_{k,n}^{(t)})-\gqkt(u;\x_{k,n}^{(t)}) \right \|} \notag \\
    &\stackrel{(b)}{\leq}\frac{1}{N_{k}^{(t)}}\sum_{n=1}^{N_{k}^{(t)}}L^{\x_{k,n}^{(t)}} \norm{w-u} \notag \\
    &=L^{\boldsymbol{\mathcal{X}}_{k}^{(t)}} \norm{w-u} 
\end{align}
where \( (a) \) follows from Jensen's inequality, and \( (b) \) follows from \eqref{lipschitz grad data specific}. Likewise,
\begin{align}
    \e_{\mathcal{X}_{k}^{(t)}}(L^{\boldsymbol{\mathcal{X}}_{k}^{(t)}})^{2} &= \e_{\mathcal{X}_{k}^{(t)}} \left( \frac{1}{N_{k}^{(t)}}\sum_{n=1}^{N_{k}^{(t)}}L^{\x_{k,n}^{(t)}} \right)^{2} \notag \\
    &\stackrel{(a)}{\leq}\frac{1}{N_{k}^{(t)}}\e_{\mathcal{X}_{k}^{(t)}}\sum_{n=1}^{N_{k}^{(t)}}(L^{\x_{k,n}^{(t)}})^{2} \notag \\
    &\stackrel{(b)}{\leq}(L_{k}^{(t)})^{2}
\end{align}
where \( (a) \) follows from Jensen's inequality, and \( (b) \) follows from \eqref{lipschitz grad expectation data specific}.
\end{proof}
\subsection{The Implication of Assumption \ref{assumption lipschitz hessians}}\label{appendix lip hessian proofs}
Assumption \ref{assumption lipschitz hessians} implies for the loss Hessian under a batch of data:
\begin{equation}\label{lipschitz hessian dataset specific}
      \e_{\mathcal{X}_{k}^{(t)}}\left \| \hqkt\Big(w;\boldsymbol{\mathcal{X}}_{k}^{(t)}\Big)-\hqkt\Big(u;\boldsymbol{\mathcal{X}}_{k}^{(t)}\Big)\right \| \leq \rho_k^{(t)}\norm{w-u}
\end{equation}\qed
\begin{proof}
For the loss Hessians under a batch of data:
\begin{align}
   \e_{\mathcal{X}_{k}^{(t)}}{\left \| \hqkt(w;\boldsymbol{\mathcal{X}}_{k}^{(t)})-\hqkt(u;\boldsymbol{\mathcal{X}}_{k}^{(t)})\right \|}&=\e_{\mathcal{X}_{k}^{(t)}}\left \|\frac{1}{N_{k}^{(t)}} \sum_{n=1}^{N_{k}^{(t)}}\Big(\hqkt(w;\x_{k,n}^{(t)})-\hqkt(u;\x_{k,n}^{(t)})\Big)\right \| \notag \\
    &\stackrel{(a)}{\leq}\frac{1}{N_{k}^{(t)}}\sum_{n=1}^{N_{k}^{(t)}}\e_{x_{k,n}^{(t)}}{\left \| \hqkt(w;\x_{k,n}^{(t)})-\hqkt(u;\x_{k,n}^{(t)})\right \|} \notag \\
    &\stackrel{(b)}{\leq}\frac{1}{N_{k}^{(t)}}\sum_{n=1}^{N_{k}^{(t)}}\rho_k^{(t)} \norm{w-u} \notag \\
    &=\rho_k^{(t)} \norm{w-u} 
\end{align}
where \( (a) \) follows from Jensen's inequality, and \( (b) \) follows from \eqref{lipschitz hessian data specific}.
\end{proof}
\subsection{The Implication of Assumption \ref{assumption bounded gradients}}\label{appendix bounded proofs}
Assumption \ref{assumption bounded gradients} implies for the loss gradient under a batch of data
\begin{equation}\label{bounded grad dataset specific}
    \e_{\mathcal{X}_{k}^{(t)}}\left \| \gqkt\Big(w;\boldsymbol{\mathcal{X}}_{k}^{(t)}\Big)\right \| \leq  B_{k}^{(t)}
\end{equation}\qed
\begin{proof}
The bound for the norm of the stochastic gradients constructed using a batch is derived as follows:
\begin{align}
    \e_{\mathcal{X}_{k}^{(t)}}{\left \| \gqkt(w;\boldsymbol{\mathcal{X}}_{k}^{(t)})\right \|} &=\e_{\mathcal{X}_{k}^{(t)}} \left \|\frac{1}{N_{k}^{(t)}} \sum_{n=1}^{N_{k}^{(t)}}\gqkt(w;\x_{k,n}^{(t)})\right \| \notag \\
    &\stackrel{(a)}{\leq}\frac{1}{N_{k}^{(t)}}\sum_{n=1}^{N_{k}^{(t)}}\e_{x_{k,n}^{(t)}}{\left \| \gqkt(w;\x_{k,n}^{(t)})\right \|} \notag \\
    &\stackrel{(b)}{\leq}\frac{1}{N_{k}^{(t)}}\sum_{n=1}^{N_{k}^{(t)}}B_k^{(t)} \notag \\
    &=B_k^{(t)} 
\end{align}
where \( (a) \) follows from Jensen's inequality, and \( (b) \) follows from \eqref{bounded grad data specific}.
\end{proof}
\subsection{The Implication of Assumption \ref{assumption fourth order data expectation}}\label{appendix sigma g h proofs}
Assumption \ref{assumption fourth order data expectation} implies for the gradient and the Hessian under a batch of data:
\begin{align}
    &\e_{\mathcal{X}_{k}^{(t)}} {\left \| \gqkt\Big(w;\boldsymbol{\mathcal{X}}_{k}^{(t)}\Big)-\gjkt(w) \right \|}^{4}\leq \frac{3\sigg^{4}}{(N_{k}^{(t)})^2} \label{sigma g dataset}\\
    &\e_{\mathcal{X}_{k}^{(t)}} {\left \| \hqkt\Big(w;\boldsymbol{\mathcal{X}}_{k}^{(t)}\Big)-\hjkt(w) \right \|}^{4}\leq \frac{3\sigh^{4}}{(N_{k}^{(t)})^{2}} \label{sigma h dataset}
\end{align}\qed
\begin{proof}
We apply induction on \( \nkt \) \citep{Vlaski2020SecondOrder}. For \( \nkt=1 \), expression \eqref{sigma g dataset} trivially holds since \eqref{sigma g data specific} is a tighter bound than \eqref{sigma g dataset}. Now assume that \eqref{sigma g dataset} holds for \( \nkt-1 \). Define:
\begin{align}\label{r definitions}
    \rr(w)&\triangleq \gqkt(w;\x_{k}^{(t)})-\gjkt(w) \notag \\
    \rr_{\nkt}(w)&\triangleq\gqkt(w;\boldsymbol{\mathcal{X}}_{k}^{(t)})-\gjkt(w) 
    =\frac{1}{\nkt}\sum_{n=1}^{\nkt}\gqkt(w;\x_{k,n}^{(t)})-\gjkt(w) 
\end{align}
Then, we get:
\begin{align}
    \e\norm{\rr_{\nkt}(w)}^{4} =& \e { \left \|\frac{\nkt-1}{\nkt}\rr_{\nkt-1}(w)+\frac{1}{\nkt}\rr(w)\right \|}^{4} \notag \\
    =&\e\left({\left \|\frac{\nkt-1}{\nkt}\rr_{\nkt-1}(w)+\frac{1}{\nkt}\rr(w)\right \|}^{2}\right)^{2} \notag \\
    =&\e\left(\frac{(\nkt-1)^{2}}{(\nkt)^{2}}\norm{\rr_{\nkt-1}(w)}^{2}+\frac{1}{(\nkt)^{2}}\norm{\rr(w)}^{2}+2\frac{\nkt-1}{(\nkt)^{2}}\rr_{\nkt-1}(w)^{\T} \rr(w)\right)^{2} \notag \\
    \stackrel{(a)}{=}&\frac{(\nkt-1)^{4}}{(\nkt)^{4}}\e\Big[\norm{\rr_{\nkt-1}(w)}^{4}\Big]+\frac{1}{(\nkt)^{4}}\e\Big[\norm{\rr(w)}^{4}\Big]+\frac{4(\nkt-1)^{2}}{(\nkt)^{4}}\e\Big[(\rr_{\nkt-1}(w)^{\T}\rr(w))^{2}\Big] \notag \\ &+\frac{2(\nkt-1)^{2}}{(\nkt)^{4}} \e\Big[\norm{\rr_{\nkt-1}(w)}^{2}\norm{\rr(w)}^{2}\Big] \notag \\
    \stackrel{(b)}{\leq}&\frac{(\nkt-1)^{4}}{(\nkt)^{4}}\e\Big[\norm{\rr_{\nkt-1}(w)}^{4}\Big]+\frac{1}{(\nkt)^{4}}\e\Big[\norm{\rr(w)}^{4}\Big]+\frac{6(\nkt-1)^{2}}{(\nkt)^{4}}\e\Big[\norm{\rr_{\nkt-1}(w)}^{2}\norm{\rr(w)}^{2}\Big] \notag \\
    \stackrel{(c)}{=}&\frac{(\nkt-1)^{4}}{(\nkt)^{4}}\e\Big[\norm{\rr_{\nkt-1}(w)}^{4}\Big]+\frac{1}{(\nkt)^{4}}\e\Big[\norm{\rr(w)}^{4}\Big]+\frac{6(\nkt-1)^{2}}{(\nkt)^{4}}\e\Big[\norm{\rr_{\nkt-1}(w)}^{2}\Big]\e\Big[\norm{\rr(w)}^{2}\Big] \notag \\
    \stackrel{(d)}{\leq}&\frac{(\nkt-1)^{4}}{(\nkt)^{4}}\frac{3\sigg^{4}}{(\nkt-1)^{2}}+\frac{1}{(\nkt)^{4}}\sigg^{4}+\frac{6(\nkt-1)^{2}}{(\nkt)^{4}}\frac{\sigg^{2}}{(\nkt-1)}\sigg^{2} \notag \\
    =&\frac{\sigg^{4}}{(\nkt-1)^{2}}\left(\frac{3(\nkt-1)^{2}}{(\nkt)^{2}}+\frac{1}{(\nkt)^{2}}+\frac{6(\nkt-1)}{(\nkt)^{2}} \right) \notag \\
    =&\frac{\sigg^{4}}{(\nkt)^{2}}\left(\frac{3(\nkt)^{2}-2}{(\nkt)^{2}} \right) \notag \\
    \leq& \frac{3\sigg^{4}}{(\nkt)^{2}} 
\end{align}
where \( (a) \) follows from expansion of the square and dropping the cross-terms that are zero due to the independence assumption on the data,\( (b) \) follows from Cauchy-Schwarz, \( (c) \) follows from independence assumption on the data, and \( (d) \) follows from the induction hypothesis, \eqref{sigma g data specific}, and the following variance reduction formula:
\begin{equation}
    \e\norm{\rr_{\nkt-1}(w)}^{2}=\frac{1}{\nkt-1}\e\norm{\rr(w)}^{2}
\end{equation}
For proving \eqref{sigma h dataset}, just replacing the gradients with the Hessians in \eqref{r definitions} is enough.
\end{proof}
\section{Alternative MAML Objective Proofs}
\subsection{Proof of \eqref{adjusted objective grad formula}}\label{appendix adj grad proof}
Recall the definition of the adjusted objective:
\begin{align}
    \jkbh(w) = \etk\exin\rjkt\Big(w-\al\rgqkt(w;\rxin)\Big)
\end{align}
The gradient corresponding to this objective is:
\begin{align}
    \gjkbh(w) = \etk\exin \left [ \Big(I-\al\rhqkt(w;\rxin)\Big)\rgjkt\Big(w-\al\rgqkt(w;\rxin)\Big) \right ]
\end{align}
Expectation of the stochastic MAML gradient is given by:
\begin{align}
    \e \nabla \overline{Q_{k}}(w)   &= \e \Big [ \averagetasks \Big(I-\al \rhqkt(w;\rxin)\Big)\rgqkt\Big(w-\al \rgqkt(w;\rxin)\ ;\ \rxout\Big) \Big ] \notag \\
    &\stackrel{(a)}{=} \e \Big [\Big(I-\al \rhqkt(w;\rxin)\Big)\rgqkt\Big(w-\al \rgqkt(w;\rxin)\ ;\ \rxout\Big) \Big ] \notag \\
    &\stackrel{(b)}{=}\etk \Big[ \exin \Big[ \exout \Big [ \Big(I-\al \rhqkt(w;\rxin)\Big)\rgqkt\Big(w-\al \rgqkt(w;\rxin)\ ;\ \rxout\Big) \  | \ \rxin \Big] \Big] \Big] \notag \\
    &=\etk \Big[ \exin \Big[  \Big(I-\al \rhqkt(w;\rxin)\Big) \exout \Big [ \rgqkt\Big(w-\al \rgqkt(w;\rxin)\ ;\ \rxout\Big)   \  | \ \rxin \Big] \Big] \Big] \notag \\
    &\stackrel{(c)}{=} \etk \Big[ \exin \left [ \Big(I-\al\rhqkt(w;\rxin)\Big)\rgjkt\Big(w-\al\rgqkt(w;\rxin)\Big) \right ] \Big] 
\end{align}
where \( (a) \) follows from the \emph{i.i.d.} assumption on the batch of tasks, \( (b) \) follows from conditioning on \( \rxin \), and \( (c) \) follows from the relation between loss functions and stochastic risks.
\subsection{Proof of Lemma \ref{perturbation analysis formula} }\label{appendix perturb analysis}
The disagreement between \eqref{real risk} and \eqref{adjusted objective formula} is:
\begin{align}
        {\left | \jkb(w)-\jkbh(w) \right |}&=\left | \e\left[\rjkt(\widetilde{\w}_{1})-\rjkt(\widetilde{\w}_{2})\right]\right | \stackrel{(a)}{\leq}\e{\left | \rjkt(\widetilde{\w}_{1})-\rjkt(\widetilde{\w}_{2}) \right |} 
\end{align}

where $\widetilde{\w}_{1} \triangleq w-\al \rgjkt(w)$ , $\widetilde{\w}_{2} \triangleq w-\al \rgqkt(w;\rxin)$ and \( (a) \) follows from Jensen's inequality. Lipschitz property of the gradient (Assumption \ref{assumption lipshtiz gradient}) implies: 
\begin{align}
    \rjkt(\widetilde{\w}_{1})- \rjkt(\widetilde{\w}_{2}) &\leq (\rgjkt(\widetilde{\w}_{2}))^{\T}(\widetilde{\w}_{1}-\widetilde{\w}_{2})+\frac{L}{2}\norm{\widetilde{\w}_{1}-\widetilde{\w}_{2}}^{2} \\
    \rjkt(\widetilde{\w}_{2})- \rjkt(\widetilde{\w}_{1}) &\leq (\rgjkt(\widetilde{\w}_{1}))^{\T}(\widetilde{\w}_{2}-\widetilde{\w}_{1})+\frac{L}{2}\norm{\widetilde{\w}_{1}-\widetilde{\w}_{2}}^{2}
\end{align}

Combining the inequalities yields:

\begin{align}
     \e{\left | \rjkt(\widetilde{\w}_{1})- \rjkt(\widetilde{\w}_{2}) \right |}&\leq\e\left[\frac{L}{2}\norm{\widetilde{\w}_{1}-\widetilde{\w}_{2}}^{2}+\max\Big\{\norm{\rgjkt(\widetilde{\w}_{1})},\norm{\rgjkt(\widetilde{\w}_{2})}\Big\}\norm{\widetilde{\w}_{1}-\widetilde{\w}_{2}}\right] \notag \\
     &\stackrel{(a)}{\leq} \e\left[\frac{L}{2}\norm{\widetilde{\w}_{1}-\widetilde{\w}_{2}}^{2}+B\norm{\widetilde{\w}_{1}-\widetilde{\w}_{2}}\right] \notag \\
     &\stackrel{(b)}{\leq} \frac{\al^{2}L}{2}\e\left[{\left \| \rgjkt(w)-\rgqkt(w;\rxin) \right \|}^{2}\right]+B\al\e\left[{\left \| \rgjkt(w)-\rgqkt(w;\rxin)\right \|}\right] \notag \\
     &\stackrel{(c)}{\leq} \frac{\al^{2}L\sigg^{2}}{2\xinsize}+\frac{B\al\sigg}{\sqrt{\xinsize}}
\end{align}
where \( (a) \) follows from Assumption \ref{assumption bounded gradients}, \( (b) \) follows from inserting \( \widetilde{\w}_1 \) and \( \widetilde{\w}_2 \) expressions, and \( (c) \) follows from Assumption \ref{assumption fourth order data expectation}.

\subsection{Proof of Lemma \ref{perturbation analysis gradient formula} }\label{appendix perturb grad analysis}
Recall \eqref{real risk grad} and \eqref{adjusted objective grad formula}:
\begin{align}
    \gjkb (w) &= \etk\left[(I-\al \rhjkt(w)) \rgjkt(w-\al \rgjkt(w)) \right] \notag \\
    &=\etk \left [\rgjkt(w-\al \rgjkt(w)) - \al \rhjkt(w)\rgjkt(w-\al \rgjkt(w)) \right]  \\
    \gjkbh(w) &= \etk\exin\left[ (I-\al\rhqkt(w;\rxin))\rgjkt(w-\al\rgqkt(w;\rxin))\right] \notag \\
    &= \etk \exin \left [\rgjkt(w-\al\rgqkt(w;\rxin)) - \al\rhqkt(w;\rxin)\rgjkt(w-\al\rgqkt(w;\rxin)) \right] 
\end{align}
The norm of the disagreement then follows:
\begin{align}
    \left \| \gjkb (w) - \gjkbh (w) \right \| =& \Big \| \e \Big [\rgjkt(w-\al \rgjkt(w))- \al \rhjkt(w)\rgjkt(w-\al \rgjkt(w)) \notag \\  &- \rgjkt(w-\al\rgqkt(w;\rxin)) + \al\rhqkt(w;\rxin)\rgjkt(w-\al\rgqkt(w;\rxin))  \Big ]  \Big \| \notag \\
    \stackrel{(a)}{\leq}& \e \Big [ \Big \| \rgjkt(w-\al \rgjkt(w))- \rgjkt(w-\al\rgqkt(w;\rxin)) \notag \\ &+ \al\rhqkt(w;\rxin)\rgjkt(w-\al\rgqkt(w;\rxin)) - \al \rhjkt(w)\rgjkt(w-\al \rgjkt(w))  \Big \| \Big ] \notag \\
    \stackrel{(b)}{\leq}& \e \Big [ \left \| \rgjkt(w-\al \rgjkt(w))- \rgjkt(w-\al\rgqkt(w;\rxin)) \right \| \notag \\ &+\al \left \| \rhqkt(w;\rxin)\rgjkt(w-\al\rgqkt(w;\rxin)) -  \rhjkt(w)\rgjkt(w-\al \rgjkt(w))  \right \|  \Big ]  \label{gradient disagreement full term}
\end{align}
where \( (a) \) follows from applying Jensen's inequality and rearranging the terms, and \( (b) \) follows from applying triangle inequality. We bound the terms in \eqref{gradient disagreement full term} separately. For the first term we have:
\begin{align}
   \e \left [ \left \| \rgjkt(w-\al \rgjkt(w))- \rgjkt(w-\al\rgqkt(w;\rxin)) \right \| \right ]  &\stackrel{(a)}{\leq} L \al \e  \left [ \left \| \rgjkt(w)-\rgqkt(w;\rxin) \right \|  \right] \notag \\
   &\stackrel{(b)}{\leq} L\al \frac{\sigg}{\sqrt{\xinsize}}
\end{align}
where \( (a) \) follows from Assumption \ref{assumption lipshtiz gradient}, and \( (b) \) follows from Assumption \ref{assumption fourth order data expectation}.

Rewriting the second term in \eqref{gradient disagreement full term}:
\begin{align}
   \e \Big [ \Big \| &\rhqkt(w;\rxin)\rgjkt(w-\al\rgqkt(w;\rxin)) -  \rhjkt(w)\rgjkt(w-\al \rgjkt(w))  \Big \| \Big ] \notag \\ \stackrel{(a)}{\leq}& \e \Big [ \Big \| \rhqkt(w;\rxin)\rgjkt(w-\al\rgqkt(w;\rxin)) - \rhjkt(w) \rgjkt(w-\al \rgqkt (w;\rxin)) \Big \| \notag \\ &+ \Big \| \rhjkt(w) \rgjkt(w-\al \rgqkt (w;\rxin)) - \rhjkt(w)\rgjkt(w-\al \rgjkt(w)) \Big \| \Big] \label{gradient disagreement second term}
\end{align}
where \( (a) \) follows from adding and subtracting the term \( \rhjkt(w) \rgjkt(w-\al \rgqkt (w;\rxin)) \) and applying the triangle inequality. We bound the terms in \eqref{gradient disagreement second term} separately. For the first term:
\begin{align}
    \e \Big [ \Big \| \rhqkt(w;\rxin)&\rgjkt(w-\al\rgqkt(w;\rxin)) - \rhjkt(w) \rgjkt(w-\al \rgqkt (w;\rxin)) \Big \| \Big ] \notag \\
    &\stackrel{(a)}{\leq} \e \Big [ \Big \| \rhqkt(w;\rxin)-\rhjkt(w) \Big \| \Big \| \rgjkt(w-\al\rgqkt(w;\rxin)) \Big \| \Big ] \notag \\
    &\stackrel{(b)}{\leq} \e \Big [ \Big \| \rhqkt(w;\rxin)-\rhjkt(w) \Big \| \Big] B \notag \\
    &\stackrel{(c)}{\leq} B \frac{\sigh}{\sqrt{\xinsize}}
\end{align}
where \( (a) \) follows from sub-multiplicity of the norm, \( (b) \) follows from Assumption \ref{assumption bounded gradients}, and \( (c) \) follows from Assumption \ref{assumption fourth order data expectation}. For the second term in \eqref{gradient disagreement second term}:
\begin{align}
    \e \Big [ \Big \| \rhjkt(w) &\rgjkt(w-\al \rgqkt (w;\rxin)) - \rhjkt(w)\rgjkt(w-\al \rgjkt(w)) \Big \| \Big] \notag \\
    &\stackrel{(a)}{\leq}  \e \Big [\Big \| \rhjkt(w) \Big \| \Big \| \rgjkt(w-\al \rgqkt (w;\rxin))-\rgjkt(w-\al \rgjkt(w)) \Big \| \Big] \notag \\
    &\stackrel{(b)}{\leq} L \e \Big [ \Big \| \rgjkt(w-\al \rgqkt (w;\rxin))-\rgjkt(w-\al \rgjkt(w)) \Big \| \Big] \notag \\
    &\stackrel{(c)}{\leq} \al L^2 \e \Big [ \Big \| \rgjkt(w)-\rgqkt(w;\rxin) \Big \| \Big ] \notag \\
    &\stackrel{(d)}{\leq} \al L^2 \frac{\sigg}{\sqrt{\xinsize}}
\end{align}
where \( (a) \) follows from sub-multiplicity of the norm, \( (b) \) and \( (c) \) follow from Assumption \ref{assumption lipshtiz gradient}, and \( (d) \) follows from Assumption \ref{assumption fourth order data expectation}. Combining the results completes the proof.
 
\subsection{Proof of Lemma \ref{adj objective bounded grad formula}}\label{appendix bounded gradient adj objective proof}
Recall the formula for the gradient of the adjusted objective \eqref{adjusted objective grad formula}:
\begin{align}
{\left \| \gjkbh(w) \right \|}&={\left \| \e \left [(I-\al \rhqkt(w;\rxin))\rgqkt\Big(w-\al \rgqkt(w;\rxin);\rxout\Big)\right ]\right \|} \notag \\
&\stackrel{(a)}{\leq} \e {\left \| (I-\al \rhqkt(w;\rxin))\rgqkt\Big(w-\al \rgqkt(w;\rxin);\rxout\Big) \right \|} \notag \\
&\stackrel{(b)}{\leq} \e \left [{\left \| (I-\al \rhqkt(w;\rxin))\right \|} {\left \| \rgqkt\Big(w-\al \rgqkt(w;\rxin);\rxout\Big)\right \|}\right ] \notag \\
&\stackrel{(c)}{\leq}   \etk\left [\exin \left [ \exout \left [{\left \| (I-\al \rhqkt(w;\rxin)) \right \| } {\left \| \rgqkt\Big(w-\al \rgqkt(w;\rxin);\rxout\Big) \right \|} \ | \rxin \right] \right ]\right ] \notag \\
&\stackrel{(d)}{\leq} \etk \left [\exin \left[{\left \| (I-\al \rhqkt(w;\rxin)) \right \|}B \right] \right] \notag \\
&\stackrel{(e)}{\leq} (1+\al L)B
\end{align}
where \( (a) \) follows from Jensen's inequality, \( (b) \) follows from sub-multiplicity of the norm, \( (c) \) follows from conditioning on \( \rxin \), \( (d) \) follows from Assumption \ref{assumption bounded gradients}, and \( (e) \) follows from Assumption \ref{assumption lipshtiz gradient}.

\subsection{Proof of Lemma \ref{adj objective lipschitz formula}}\label{appendix adj obj lipschitz}
Define the following variables:
\begin{align}{}
     \widetilde{\w}_2 &\triangleq w - \al \rgqkt(w;\rxin) \\
     \widetilde{\boldsymbol{u}}_2 &\triangleq u - \al \rgqkt(u;\rxin)
\end{align}
Recall the formula for the gradient of the adjusted objective \eqref{adjusted objective grad formula}:
\begin{align}
    \gjkbh(w) &= \e\left[(I-\al \rhqkt(w;\rxin))\rgqkt(\widetilde{\w}_2;\rxout)\right ] \notag \\
    &= \e \left [\rgqkt(\widetilde{\w}_2;\rxout)-\al \rhqkt(w;\rxin)\rgqkt(\widetilde{\w}_2;\rxout) \right ]  \\
    \gjkbh(u) &= \e\left[(I-\al \rhqkt(u;\rxin))\rgqkt(\widetilde{\boldsymbol{u}}_2;\rxout)\right ] \notag \\
    &= \e \left [\rgqkt(\widetilde{\boldsymbol{u}}_2;\rxout)-\al \rhqkt(u;\rxin)\rgqkt(\widetilde{\boldsymbol{u}}_2;\rxout)\right ] 
\end{align}
Bounding the disagreement:
\begin{align}{}
\Big \| \gjkbh(w)-\gjkbh(u) \Big \| =& \Big \| \e \Big [\rgqkt(\widetilde{\w}_2;\rxout)-\rgqkt(\widetilde{\boldsymbol{u}}_2;\rxout)-\al \Big(\rhqkt(w;\rxin)\rgqkt(\widetilde{\w}_2;\rxout) \notag \\ &-\rhqkt(u;\rxin)\rgqkt(\widetilde{\boldsymbol{u}}_2;\rxout)\Big) \Big ] \Big \|  \notag \\
\stackrel{(a)}{\leq}& \e\Big [\Big \| \rgqkt(\widetilde{\w}_2;\rxout)-\rgqkt(\widetilde{\boldsymbol{u}}_2;\rxout)-\al \Big(\rhqkt(w;\rxin)\rgqkt(\widetilde{\w}_2;\rxout) \notag \\ &- \rhqkt(u;\rxin)\rgqkt(\widetilde{\boldsymbol{u}}_2;\rxout)\Big)\Big \|\Big ] \notag \\
\stackrel{(b)}{\leq}& \e \left [\Big \| \rgqkt(\widetilde{\w}_2;\rxout)-\rgqkt(\widetilde{\boldsymbol{u}}_2;\rxout) \Big \| \right ] +\al\e\Big [ \Big \| \rhqkt(w;\rxin)\rgqkt(\widetilde{\w}_2;\rxout) \notag \\&- \rhqkt(u;\rxin)\rgqkt(\widetilde{\boldsymbol{u}}_2;\rxout)\Big \|\Big ] \label{gradlip two exp}
\end{align}

where \( (a) \) follows from Jensen's inequality, and \( (b) \) follows from the triangle inequality. We bound the terms in \eqref{gradlip two exp} separately. For the first term,
\begin{align}{}
 \e\left [{\left \| \rgqkt(\widetilde{\w}_2;\rxout)-\rgqkt(\widetilde{\boldsymbol{u}}_2;\rxout) \right \|} \right ] &\stackrel{(a)}{\leq} \e \left[ L^{\rxout} \norm{\widetilde{\w}_2-\widetilde{\boldsymbol{u}}_2} \right] \notag \\
 &\stackrel{(b)}{\leq} \e \left[ L^{\rxout} \Big(\norm{w-u}+\al {\left \| \rgqkt(w;\rxin)-\rgqkt(u;\rxin) \right \|}\Big) \right ] \notag \\
 &\stackrel{(c)}{\leq} \e \left[ L^{\rxout} \Big(\norm{w-u}+\al L^{\rxin} \norm{w-u}\Big) \right] \label{first lip bound to be used later} \\
 &\stackrel{(d)}{\leq}  L(1+\al L) \norm{w-u} 
\end{align}
where \( (a) \) follows from Assumption \ref{assumption lipshtiz gradient}, \( (b) \) follows from replacing \( \widetilde{\w}_2,\widetilde{\boldsymbol{u}}_2 \) and applying triangle inequality, \( (c) \) follows from Assumption \ref{assumption lipshtiz gradient}, \( (d) \) follows from the independence assumption on \( \rxin,\rxout \) and taking the expectation. For the second term we have:
\begin{align}
    \e \Big \| \rhqkt(w;\rxin)&\rgqkt(\widetilde{\w}_2;\rxout)-\rhqkt(u;\rxin)\rgqkt(\widetilde{\boldsymbol{u}}_2;\rxout) \Big \|  \notag \\
    \stackrel{(a)}{\leq}& \e\left [{\left \| \rhqkt(w;\rxin)\rgqkt(\widetilde{\w}_2;\rxout)-\rhqkt(w;\rxin)\rgqkt(\widetilde{\boldsymbol{u}}_2;\rxout) \right \|} \right ] \notag \\
    &+ \e \left [ \left \| \rhqkt(w;\rxin)\rgqkt(\widetilde{\boldsymbol{u}}_2;\rxout)-\rhqkt(u;\rxin)\rgqkt(\widetilde{\boldsymbol{u}}_2;\rxout) \right \| \right ] \label{lip proof same hessian and hessian dif} 
\end{align}{}

where \( (a) \) follows from adding and subtracting the same term and triangle inequality. For the first term in \eqref{lip proof same hessian and hessian dif}, we have:

\begin{align}
    \e \Big\| \rhqkt(w;\rxin)\rgqkt(&\widetilde{\w}_2;\rxout)-\rhqkt(w;\rxin)\rgqkt(\widetilde{\boldsymbol{u}}_2;\rxout) \Big\| \notag \\
    &\stackrel{(a)}{\leq} \e\left [{\left \| \rhqkt(w;\rxin)\right \|}{\left \| \rgqkt(\widetilde{\w}_2;\rxout)-\rgqkt(\widetilde{\boldsymbol{u}}_2;\rxout)\right \| }\right ] \notag \\
    &\stackrel{(b)}{\leq} \e\left [(L^{\rxin})\Big(L^{\rxout}(1+\al L^{\rxin})\norm{w-u}\Big)\right ] \notag \\
    &\stackrel{(c)}{\leq} L^{2}(1+\al L) \norm{w-u}
\end{align}{}

where \( (a) \) follows from sub-multiplicity of the norm, \( (b) \) follows from Assumption \ref{assumption lipshtiz gradient} and \eqref{first lip bound to be used later}, \( (c) \) follows from the independence assumption on \( \rxin,\rxout \) and taking the expectation.

For the second term in \eqref{lip proof same hessian and hessian dif}, we have:
 
 \begin{align}
     \e \Big \| \rhqkt(w;\rxin)&\rgqkt(\widetilde{\boldsymbol{u}}_2;\rxout)-\rhqkt(u;\rxin)\rgqkt(\widetilde{\boldsymbol{u}}_2;\rxout)\Big \| \notag \\
     &\stackrel{(a)}{\leq} \e\Big [\Big \| \rhqkt(w;\rxin)-\rhqkt(u;\rxin)\Big \| \Big \| \rgqkt(\widetilde{\boldsymbol{u}}_2;\rxout)\Big \|\Big ] \notag \\
     &\stackrel{(b)}{\leq} \etk\left [\exin\left [ \exout \left [{\left \| \rhqkt(w;\rxin)-\rhqkt(u;\rxin) \right \| } {\left \| \rgqkt(\widetilde{\boldsymbol{u}}_2;\rxout) \right \|} \ |  \rxin \right ]\right ] \right] \notag \\
     &\stackrel{(c)}{\leq} \etk\left [\exin\left [{\left \| \rhqkt(w;\rxin)-\rhqkt(u;\rxin) \right \| } B \right] \right] \notag \\
     &\stackrel{(d)}{\leq} \rho B \norm{w-u}
 \end{align}{}
 
where \( (a) \) follows from sub-multiplicity of the norm, \( (b) \) follows from conditioning on \( \rxin  \), \( (c) \) follows from Assumption \ref{assumption bounded gradients} and \( (d) \) follows from Assumption \ref{assumption lipschitz hessians}. Combining the results completes the proof.

\subsection{Proof of Lemma \ref{adj objective noise bound}}\label{appendix adj obj noise process}
We will first prove three intermediate lemmas, then conclude the proof.

First defining the task-specific meta-gradient and task-specific meta-stochastic gradient:
\begin{align}
    &\rgqktb(w) \triangleq(I-\al \rhqkt(w;\rxin))\rgqkt\Big(w-\al \rgqkt(w;\rxin);\rxout\Big) \label{c1 proof rewriting maml grad} \\
        &\rgjktb (w) \triangleq (I-\al \rhjkt(w)) \rgjkt(w-\al \rgjkt(w)) 
\end{align}{}
\begin{lemma}\label{lemma:c1}
Under assumptions \ref{assumption lipshtiz gradient},\ref{assumption bounded gradients},\ref{assumption fourth order data expectation}, for each agent \( k \), the disagreement between \( \rgqktb ( \cdot ) \) and \( \rgjktb ( \cdot ) \) is bounded in expectation, namely, for any \( w \in \mathds{R}^{M} \):
\begin{equation}
\begin{aligned}{}
    \e\norm{\rgqktb(w)-\rgjktb(w)}^{2}\leq C_{1}^{2}
    \end{aligned}
\end{equation}
where \( C_{1}^{2} \triangleq 6(1+\al L)^{2}\sigg^{2}(\frac{1}{\xosize}+\frac{L^{2}\al^{2}}{\xinsize})+ \frac{6\al^{2}\sigh^{2}}{\xinsize}(B^{2}+\frac{\sigg^{2}}{\xosize})+\frac{9\al^{4}}{\xinsize^{2}}(\sigh^{4}+L^{4}\sigg^{4}) \) is a non-negative constant.

\begin{proof}
Defining the error terms:
\begin{align}
    &\erhxt \triangleq \al \rhjkt(w)- \al \rhqkt(w;\rxin) \\
    &\ergot \triangleq \rgqkt\Big(w-\al\rgqkt(w;\rxin);\rxout\Big)-\rgjkt(w-\al\rgqkt(w;\rxin))\\
    &\ergxt \triangleq \rgjkt(w-\al\rgqkt(w;\rxin))-\rgjkt(w-\al\rgjkt(w))
\end{align}{}
Rewriting \eqref{c1 proof rewriting maml grad}:
\begin{align}
    \rgqktb(w)=&(I-\al \rhjkt(w)+\erhxt)(\rgjkt(w-\al \rgjkt(w))+\ergot+\ergxt) \notag \\
    \rgqktb(w)-\rgjktb(w)=&(I-\al \rhjkt(w))\ergot+(I-\al \rhjkt(w))\ergxt+\erhxt \rgjkt(w-\al \rgjkt(w))\notag \\&+\erhxt\ergot+\erhxt\ergxt 
\end{align}{}
Bounding the disagreement:
\begin{align}
    \norm{\rgqktb(w)-\rgjktb(w)}\stackrel{(a)}{\leq}& \norm{(I-\al \rhjkt(w))}\norm{\ergot}+\norm{(I-\al \rhjkt(w))}\norm{\ergxt} \notag \\&+\norm{\erhxt}\norm{\rgjkt(w-\al \rgjkt(w))}+\norm{\erhxt}\norm{\ergot}+\norm{\erhxt}\norm{\ergxt}
  \notag \\\stackrel{(b)}{\leq}& \norm{(I-\al \rhjkt(w))}\norm{\ergot}+\norm{(I-\al \rhjkt(w))}\norm{\ergxt} \notag \\&+ \norm{\erhxt}\norm{\rgjkt(w-\al \rgjkt(w))}+\norm{\erhxt}\norm{\ergot}+\frac{\norm{\erhxt}^{2}}{2}+\frac{\norm{\ergxt}^{2}}{2}  
\end{align}{}
where \( (a) \) follows from sub-multiplicity of the norm and triangle inequality, \( (b) \) follows from $ab\leq \frac{a^{2}+b^{2}}{2}$ .

Taking the square of the norm and using $(\sum_{i=1}^{6}x_{i})^{2}\leq 6(\sum_{i=1}^{6}x_{i}^{2})$ (Cauchy-Schwarz) yield:
\begin{align}{}
    \norm{\rgqktb(w)-&\rgjktb(w)}^{2}\leq6\norm{(I-\al \rhjkt(w))}^{2}\norm{\ergot}^{2}+6\norm{(I-\al \rhjkt(w))}^{2}\norm{\ergxt}^{2} \notag \\&+6\norm{\erhxt}^{2}\norm{\rgjkt(w-\al \rgjkt(w))}^{2}+6\norm{\erhxt}^{2}\norm{\ergot}^{2}+3\norm{\erhxt}^{4}+3\norm{\ergxt}^{4}  
\end{align}
Taking the expectation with respect to inner and outer batch of data yields:
\begin{align}
    \eds\norm{\rgqktb&(w)-\rgjktb(w)}^{2}\leq6\norm{(I-\al \rhjkt(w))}^{2}\eds\Big[\norm{\ergot}^{2}\Big]\notag \\ &+6 \norm{(I-\al \rhjkt(w))}^{2}  \eds\Big[\norm{\ergxt}^{2}\Big]+6\eds\Big[\norm{\erhxt}^{2}\Big]\norm{\rgjkt(w-\al \rgjkt(w))}^{2}\notag \\&+ 6\eds\Big[\norm{\erhxt}^{2}\norm{\ergot}^{2}\Big]+3\eds\Big[\norm{\erhxt}^{4}\Big]+3\eds\Big[\norm{\ergxt}^{4}\Big]  
    \label{c1 proof long formula}
\end{align}

We bound the terms of \eqref{c1 proof long formula} one by one.
\begin{align}
    {\Big \| (I-\al \rhjkt(w))\Big \|}    &\stackrel{(a)}{\leq} (1+\al L)  \\
    {\Big \| (I-\al \rhjkt(w)) \Big \|}^{2}&\leq(1+\al L)^{2} \label{I-hessian bound}
\end{align}{}
where \( (a) \) follows from Assumption \ref{assumption lipshtiz gradient}.

\begin{align}{}
    \eds\norm{\ergot}^{2} &= \eds{\left\|\rgqkt\Big(w-\al\rgqkt(w;\rxin);\rxout\Big)-\rgjkt(w-\al\rgqkt(w;\rxin))\right \|}^{2} \notag \\
    &\stackrel{(a)}{=}\eds{\left \| \rgqkt(\widetilde{\w}_2;\rxout)-\rgjkt(\widetilde{\w}_2)\right \|}^{2} \notag \\
    &\stackrel{(b)}{\leq} \frac{\sigg^{2}}{\xosize} \label{ergo2 bound}
\end{align}
where \( (a) \) follows from defining \( \widetilde{\w}_2 \triangleq w-\al\rgqkt(w;\rxin) \), and \( (b) \) follows from conditioning on \( \rxin \) and Assumption \ref{assumption fourth order data expectation}.

\begin{align}{}
    \norm{\ergxt} &= \left \| \rgjkt(w-\al\rgqkt(w;\rxin))-\rgjkt(w-\al\rgjkt(w)) \right \| \notag \\
    &\stackrel{(a)}{\leq} \al L \left \| \rgjkt(w)-\rgqkt(w;\rxin) \right \|  \\
    \norm{\ergxt}^{4} &\leq \al^{4} L^{4} {\left\|\rgjkt(w)-\rgqkt(w;\rxin)\right \|}^{4}  \\
    \exin\left[\norm{\ergxt}^{4}\right] &\leq \al^{4} L^{4} \exin\left[{\left \| \rgjkt(w)-\rgqkt(w;\rxin)\right \|}^{4}\right] 
    \stackrel{(b)}{\leq} \al^{4} L^{4} \frac{3\sigg^{4}}{\xinsize^{2}}  \\
    \exin\left[\norm{\ergxt}^{2}\right] &\stackrel{(c)}{\leq} \al^{2} L^{2} \exin\left[{\left \| \rgjkt(w)-\rgqkt(w;\rxin)\right \|}^{2}\right] \stackrel{(d)}{\leq} \al^{2} L^{2}\frac{\sigg^{2}}{\xinsize} \label{ergxt2 bound}
\end{align}
where \( (a) \) follows from Assumption \ref{assumption lipshtiz gradient}, \( (b) \) follows from Assumption \ref{assumption fourth order data expectation}, \( (c) \) follows from taking square and expectation of \( (a) \), and \( (d) \) follows from Assumption \ref{assumption fourth order data expectation}.

\begin{align}
    \exin\norm{\erhxt}^{4} &= \al^{4} \exin{\left \| \rhjkt(w)- \al \rhqkt(w;\rxin) \right \|}^{4} 
    \stackrel{(a)}{\leq} \al^{4}\frac{3\sigh^{4}}{\xinsize^{2}}  \\
    \exin\norm{\erhxt}^{2} &= \al^{2} \exin{\left \| \rhjkt(w)- \al \rhqkt(w;\rxin)\right \|}^{2} 
    \stackrel{(b)}{\leq} \al^{2}\frac{\sigh^{2}}{\xinsize} \label{erhxt2 bound}
\end{align}
where \( (a) \) and \( (b) \) follow from Assumption \ref{assumption fourth order data expectation}. Moreover,

\begin{align}
    \norm{\rgjkt(w-\al \rgjkt(w))}^{2}\leq B^{2}
\end{align}{}
because of Assumption \ref{assumption bounded gradients}, and

\begin{align}
     \eds\left [\norm{\erhxt}^{2}\norm{\ergot}^{2} \right] &= \exin\left[\exout\left [\norm{\erhxt}^{2}\norm{\ergot}^{2}| \rxin \right ]\right]  \notag \\
    &\stackrel{(a)}{\leq} \exin\left[\norm{\erhxt}^{2}\frac{\sigg^{2}}{\xosize}\right]  \notag \\
    &\stackrel{(b)}{\leq} \al^{2}\frac{\sigh^{2}}{\xinsize} \frac{\sigg^{2}}{\xosize} \label{expectation of inner and outer bound values placed}
\end{align}{}

where \( (a) \) follows from  \eqref{ergo2 bound}, and \( (b) \) follows from \eqref{erhxt2 bound}. Substituting the results into \eqref{c1 proof long formula} completes the proof.
\end{proof}{}

\end{lemma}{}

Defining $\jksb (w) := \jk (w-\al \gjk (w))$ where $\jk (w)=\etk [\rjkt (w)]$, we have the following two lemmas.

\begin{lemma}\label{lemma:c2}
Under assumptions \ref{assumption lipshtiz gradient},\ref{assumption bounded gradients},\ref{assumption fourth order task expectation}, for each agent \( k \), the disagreement between \( \rgjktb ( \cdot ) \) and \( \gjksb ( \cdot ) \) is bounded in expectation, namely, for any \( w \in \mathds{R}^{M} \):
\begin{equation}
    \e {\left \| \rgjktb (w)- \gjksb (w) \right \|}^{2} \leq C_{2}^{2}
    \end{equation}{}
where \( C_{2}^{2} \triangleq 8(1+\al L)^{2}(1+\al^{2}L^{2})\gamg^{2}+4B^{2}\al^{2}\gamh^{2}+2\al^{4}\gamh^{4}+16(1+\al^{4}L^{4})\gamg^{4} \) is a non-negative constant.

\begin{proof}
Recall the definitions:
\begin{align}
    \rgjktb (w) &= (I-\al\rhjkt (w))\rgjkt(w-\al\rgjkt(w))  \\
    \gjksb (w) &= (I-\al \hjk (w))\gjk (w-\al \gjk (w)) 
\end{align}
Defining the error terms:
\begin{align}
    &\erhtt \triangleq \al \hjk (w) - \al \rhjkt (w) \\
    &\ergtt \triangleq \rgjkt (w-\al \rgjkt (w)) - \gjk (w-\al \gjk (w))
\end{align}{}
Placing the new error definitions we have:
\begin{align}
    \rgjktb (w) &= (I-\al \hjk (w) + \erhtt) (\gjk (w-\al \gjk (w))+\ergtt)  \\
    \rgjktb (w) - \gjksb (w) &= (I-\al \hjk (w))\ergtt+\erhtt\gjk (w-\al \gjk (w))+\erhtt\ergtt  \\
    {\left \| \rgjktb (w) - \gjksb (w)\right \|}&\stackrel{(a)}{\leq} \norm{(I-\al \hjk (w))}\norm{\ergtt}+\norm{\erhtt}\norm{\gjk (w-\al \gjk (w))}+\norm{\erhtt}\norm{\ergtt} \notag \\
    &\stackrel{(b)}{\leq} \norm{(I-\al \hjk (w))}\norm{\ergtt}+\norm{\erhtt}\norm{\gjk (w-\al \gjk (w))}+\frac{\norm{\erhtt}^{2}}{2}+\frac{\norm{\ergtt}^{2}}{2} 
\end{align}

where \( (a) \) follows from sub-multiplicity and triangle inequality, \( (b) \) follows from \( ab\leq \frac{a^{2}+b^{2}}{2} \) . Using $(\sum_{i=1}^{4}x_{i})^{2}\leq 4(\sum_{i=1}^{4}x_{i}^{2})$ (Cauchy-Schwarz) and taking expectation yield:
\begin{align}
        \e{\left \| \rgjktb (w) - \gjksb (w) \right \|}^{2}\leq& 4 \norm{(I-\al \hjk (w))}^{2} \e\left [\norm{\ergtt}^{2}\right ]+ 4 \norm{\gjk (w-\al \gjk (w))}^{2} \e\left [\norm{\erhtt}^{2}\right ]\notag \\&+2\e\left [\norm{\erhtt}^{4}\right ]+2\e\left [\norm{\ergtt}^{4}\right ]
        \label{c2 expectation all the terms}
\end{align}{}

We bound terms in \eqref{c2 expectation all the terms} one by one. Note that
\begin{equation}
    {\left \| (I-\al \hjk(w)) \right \|}^{2}\leq(1+\al L)^{2} 
\end{equation}{}
by Assumption \ref{assumption lipshtiz gradient}, while

\begin{align}
    \norm{\ergtt}&={\left \| \rgjkt(w-\al \rgjkt (w))-\gjk(w-\al \gjk(w))\right \|} \notag \\
    &\stackrel{(a)}{=} \left \| \rgjkt(w-\al \rgjkt (w))-\rgjkt(w-\al \gjk(w)) + \rgjkt(\widetilde{w}_3)-\gjk(\widetilde{w}_3) \right \| \notag \\
    &\stackrel{(b)}{\leq} \left \| \rgjkt(w-\al \rgjkt (w))-\rgjkt(w-\al \gjk(w)) \right \| + \left \| \rgjkt(\widetilde{w}_3)-\gjk(\widetilde{w}_3) \right \| \notag \\
    &\stackrel{(c)}{\leq} \al L \left \| \rgjkt(w)-\gjk(w) \right \|+\left \| \rgjkt(\widetilde{w}_3)-\gjk(\widetilde{w}_3)\right \| \label{ergt norm inequality}
\end{align}

where (a) follows from the definition $\widetilde{w}_3 \triangleq w-\al \gjk(w)$, \( (b) \) follows from triangle inequality, and \( (c) \) follows from Assumption \ref{assumption lipshtiz gradient}.

For second-order moment of \( \ergtt \), using $(a+b)^{2}\leq 2(a^{2}+b^{2})$ (Cauchy-Schwarz) and taking expectation result in:
\begin{align}
    \etk\norm{\ergtt}^{2}&\leq 2\al^{2}L^{2}\etk\left [{\left \| \rgjkt(w)-\gjk(w)\right \|}^{2}\right ]+2\etk\left [{\left \| \rgjkt(\widetilde{w}_3)-\gjk(\widetilde{w}_3) \right \|}^{2}\right ] \notag \\
    &\stackrel{(a)}{\leq} 2\al^{2}L^{2}\gamg^{2}+2\gamg^{2} \notag \\
    &=2\gamg^{2}(1+\al^{2}L^{2}) \label{ergtt2 bound}
\end{align}

where \( (a) \) follows from Assumption \ref{assumption fourth order task expectation}.

For fourth-order moment of \( \ergtt \), using $(a+b)^{4}\leq 8(a^{4}+b^{4})$ (Cauchy-Schwarz) and taking expectation result in:
\begin{align}
    \etk\left [\norm{\ergtt}^{4}\right ]&\leq 8\al^{4}L^{4}\etk\left [{\left \| \rgjkt(w)-\gjk(w) \right \|}^{4}\right ]+8\etk\left [{\left \| \rgjkt(\widetilde{w}_3)-\gjk(\widetilde{w}_3) \right \|}^{4} \right ] \notag \\
    &\stackrel{(a)}{\leq} 8\al^{4}L^{4}\gamg^{4}+8\gamg^{4} \notag \\
    &\leq8\gamg^{4}(1+\al^{4}L^{4})
\end{align}

where \( (a) \) follows from Assumption \ref{assumption fourth order task expectation}. Also,

\begin{equation}
    \norm{\gjk (w-\al \gjk (w))}^{2}\leq B^{2}
\end{equation}{}
by Assumption \ref{assumption bounded gradients}, and

\begin{align}
    \etk\norm{\erhtt}^{4}&=\etk{\left \| \al\hjk (w) - \al \rhjkt (w) \right \|}^{4} \notag \\
    &=\al^{4}\etk{\left \| \hjk (w) - \rhjkt (w)\right \|}^{4} \notag \\
    &\stackrel{(a)}{\leq} \al^{4} \gamh^{4} \\
    \etk\norm{\erhtt}^{2} &\stackrel{(b)}{\leq} \sqrt{\etk\norm{\erhtt}^{4}} 
    \stackrel{(c)}{\leq} \al^{2}\gamh^{2} \label{erhtt2 bound}
\end{align}

where \( (a) \) follows from Assumption \ref{assumption fourth order task expectation}, \( (b) \) follows from Jensen's inequality, and \( (c) \) follows from taking square root of \( (a) \). Inserting all the results into \eqref{c2 expectation all the terms} completes the proof.
\end{proof}{}    
\end{lemma}{}
Next, we prove the last intermediate lemma.

\begin{lemma}\label{lemma:c3}
Under assumptions \ref{assumption lipshtiz gradient},\ref{assumption bounded gradients},\ref{assumption fourth order data expectation},\ref{assumption fourth order task expectation}, for each agent \( k \), the disagreement between \( \gjksb ( \cdot ) \) and \( \gjkbh ( \cdot ) \) is bounded, namely, for any \( w \in \mathds{R}^{M} \):
\begin{equation}
    {\left \| \gjksb (w) - \gjkbh (w) \right \|} \leq C_{3}
\end{equation}{}
where \( C_{3} \triangleq (1+\al L)\al L \frac{\sigg}{\sqrt{\xinsize}}+(1+\al L)^{2} \gamg+B\al \frac{\sigh}{\sqrt{\xinsize}}+ B\al\gamh+\al^{2}\frac{\sigh^{2}}{\xinsize}+\al^{2}\gamh^{2}+\al^{2}L^{2}\frac{\sigg^{2}}{\xinsize}+2(1+\al^{2}L^{2})\gamg^{2} \) is a non-negative constant.

\begin{proof}
Recall the definitions:
\begin{align}
        \gjksb (w) &= (I-\al \hjk (w))\gjk (w-\al \gjk (w))  \\
        \gjkbh (w) &= \e \left [(I-\al \rhqkt(w;\rxin))\rgqkt\Big(w-\al \rgqkt(w;\rxin);\rxout\Big)\right ] \notag \\
        &= \etk\left [\exin \left [(I-\al \rhqkt(w;\rxin))\rgjkt(w-\al\rgqkt(w;\rxin)) \right ]\right ] 
\end{align}

Defining the error terms:
\begin{align}
        \erhxt &\triangleq \al \rhjkt(w)- \al \rhqkt(w;\rxin) \\
        \erhtt &\triangleq \al \hjk (w) - \al \rhjkt (w)  \\
        \ergxt &\triangleq \rgjkt(w-\al\rgqkt(w;\rxin))-\rgjkt(w-\al\rgjkt(w))  \\
        \ergtt &\triangleq \rgjkt (w-\al \rgjkt (w)) - \gjk (w-\al \gjk (w)) 
\end{align}

Then, we can rewrite the components of the adjusted objective gradient as:
\begin{align}
    &I-\al \rhqkt(w;\rxin)=I-\al \hjk (w)+\erhxt+\erhtt  \\
    &\rgjkt(w-\al\rgqkt(w;\rxin))=\gjk (w-\al \gjk (w))+\ergxt+\ergtt 
\end{align}

and we can write the distance as:

\begin{align}
    {\left \| \gjkbh (w)-\gjksb (w) \right \|}=&\Big \| \e\Big [(I-\al \hjk (w))(\ergxt+\ergtt)+\gjk (w-\al \gjk (w)) (\erhxt+\erhtt)\notag \\ &+(\ergxt+\ergtt)(\erhxt+\erhtt) \Big] \Big \| \notag \\
    \stackrel{(a)}{\leq}& \e\Big[\Big \| (I-\al \hjk (w))(\ergxt+\ergtt)+\gjk (w-\al \gjk (w))(\erhxt+\erhtt)\notag\\&+(\ergxt+\ergtt)(\erhxt+\erhtt) \Big \| \Big] \notag \\
\stackrel{(b)}{\leq}& \norm{(I-\al \hjk (w))}\e\left [\norm{\ergxt}\right ]+\norm{(I-\al \hjk (w))}\e\left [\norm{\ergtt} \right ]\notag \\&+\norm{\gjk (w-\al \gjk (w))}\e\left [\norm{\erhxt}\right ]+\norm{\gjk (w-\al \gjk (w))}\e\left[\norm{\erhtt}\right]\notag \\&+\e\left[\norm{\erhxt}^{2}\right]+\e\left[\norm{\erhtt}^{2}\right]+\e\left[\norm{\ergxt}^{2}\right]+\e\left[\norm{\ergtt}^{2}\right] \label{c3 all the terms}
\end{align}

where \( (a) \) follows from Jensen's inequality, \( (b) \) follows from triangle inequality and $\norm{(a+b)(c+d)}\leq \norm{a}^{2}+\norm{b}^{2}+\norm{c}^{2}+\norm{d}^{2}$ (sub-multiplicity and triangle inequality).

We bound the terms in \eqref{c3 all the terms} one by one. Note that

\begin{equation}
    {\left \| (I-\al \hjk(w)) \right \|}\leq(1+\al L) 
\end{equation}{}
by Assumption \ref{assumption lipshtiz gradient}. Also,
\begin{align}
    \e\norm{\ergxt} &\stackrel{(a)}{\leq} \sqrt{\e\norm{\ergxt}^{2}} \stackrel{(b)}{\leq}\al L \frac{\sigg}{\sqrt{\xinsize}} \label{ergxt bound}
\end{align}

where \( (a) \) follows from Jensen's inequality, and \( (b) \) follows from \eqref{ergxt2 bound}. Likewise,
\begin{align}
    \e\norm{\ergtt} &\stackrel{(a)}{\leq} \al L \etk\left [{\left \| \rgjkt(w)-\gjk(w) \right \|}\right ]+\etk\left [{\left \| \rgjkt(\widetilde{w}_3)-\gjk(\widetilde{w}_3) \right \| }\right ] \notag \\
    &\stackrel{(b)}{\leq}  \al L \gamg + \gamg \notag \\
    &\leq(1+\al L)\gamg
\end{align}
where \( (a) \) follows from \eqref{ergt norm inequality} and taking the expectation, and \( (b) \) follows from Assumption \ref{assumption fourth order task expectation}. Moreover,
\begin{equation}
    \norm{\gjk (w-\al \gjk (w))}\leq B
\end{equation}
by Assumption \ref{assumption bounded gradients}, and
\begin{align}
    \e\norm{\erhxt} &\stackrel{(a)}{\leq} \sqrt{\e\norm{\erhxt}^{2}} \stackrel{(b)}{\leq}\al\frac{\sigh}{\sqrt{\xinsize}}
\end{align}
where \( (a) \) follows from Jensen's inequality, and \( (b) \) follows from \eqref{erhxt2 bound}. Also,
\begin{align}
    \e\norm{\erhtt} &\stackrel{(a)}{\leq} \sqrt{\e\norm{\erhtt}^{2}} \stackrel{(b)}{\leq} \al \gamh
\end{align}{}
where \( (a) \) follows from Jensen's inequality, and \( (b) \) follows from \eqref{erhtt2 bound}. Moreover,
\begin{equation}
    \e\norm{\erhxt}^{2}\leq\al^{2}\frac{\sigh^{2}}{\xinsize}
\end{equation}{}
by \eqref{erhxt2 bound},
\begin{equation}
    \e\norm{\erhtt}^{2}\leq \al^{2} \gamh^{2}
\end{equation}{}
by \eqref{erhtt2 bound},
\begin{equation}
    \e\norm{\ergxt}^{2}\leq \al^{2}L^{2}\frac{\sigg^{2}}{\xinsize}
\end{equation}{}
by \eqref{ergxt2 bound},
\begin{equation}
    \e\norm{\ergtt}^{2}\leq 2(1+\al^{2}L^{2})\gamg^{2}
\end{equation}{}
by \eqref{ergtt2 bound}.

Inserting all the bounds into \eqref{c3 all the terms} completes the proof.
\end{proof}{}
\end{lemma}{}
Now, combining the results of the previous three intermediate lemmas, we will prove that \( C^{2}=\frac{3}{\abs{\mathcal{S}_{k}}}(C_{1}^{2}+C_{2}^{2}+C_{3}^{2}) \), i.e.,
\begin{equation}\label{eq:appendix c expression}
    \e[\norm{\gqkb (w)-\gjkbh (w)}^{2}]\leq \frac{3}{\abs{\mathcal{S}_{k}}}(C_{1}^{2}+C_{2}^{2}+C_{3}^{2})
\end{equation}
where \( C_1, C_2 \) and \( C_3 \) expressions are given in Lemma \ref{lemma:c1}, Lemma \ref{lemma:c2} and Lemma \ref{lemma:c3}, respectively.
\begin{proof}
\begin{align}
    \e{\Big \| \gqkb (w)-\gjkbh (w) \Big \|}^{2}=& \e{\Big \| \averagetasks\rgqktb (w)-\gjkbh (w) \Big \|}^{2} \notag \\
    =& \e {\Big \| \averagetasks (\rgqktb (w)-\gjkbh (w)) \Big \|}^{2} \notag \\
    =&\frac{1}{\abs{\mathcal{S}_{k}}^{2}}\sum_{\boldsymbol{t} \in \boldsymbol{\mathcal{S}}_{k}}\e \Big[ \Big\| \rgqktb (w)-\gjkbh (w) \Big\|^{2} \Big] \notag \\ &+\frac{1}{\abs{\mathcal{S}_{k}}^{2}}\sum_{\boldsymbol{t}_{1}\neq \boldsymbol{t}_{2}}\e \Big[ (\nabla Q_{k}^{(\boldsymbol{t}_{1})}(w)-\gjkbh(w))(\nabla Q_{k}^{(\boldsymbol{t}_{2})}(w)-\gjkbh(w)) \Big] \notag \\
    \stackrel{(a)}{=}&\frac{1}{\abs{\mathcal{S}_{k}}^{2}}\sum_{\boldsymbol{t} \in \boldsymbol{\mathcal{S}}_{k}}\e\Big[ \Big \| \rgqktb (w)-\gjkbh (w) \Big \|^{2}\Big]\notag \\&+\frac{1}{\abs{\mathcal{S}_{k}}^{2}}\sum_{\boldsymbol{t}_{1}\neq \boldsymbol{t}_{2}}\e\Big[\nabla Q_{k}^{(\boldsymbol{t}_{1})}(w)-\gjkbh(w) \Big]\e\Big[\nabla Q_{k}^{(\boldsymbol{t}_{2})}(w)-\gjkbh(w)\Big] \notag \\
    \stackrel{(b)}{=}&\frac{1}{\abs{\mathcal{S}_{k}}^{2}}\sum_{\boldsymbol{t} \in \boldsymbol{\mathcal{S}}_{k}}\e\Big \| \rgqktb (w)-\gjkbh (w) \Big \|^{2} \notag \\
    =&\frac{1}{\abs{\mathcal{S}_{k}}}\e\Big \| \rgqktb (w)-\gjkbh (w) \Big \|^{2} \label{c proof connection batch proof}
\end{align}{}

where \( (a) \) follows from independence assumption on batch of tasks, and \( (b) \) follows from the definition of adjusted objective. Now, bounding the term in \eqref{c proof connection batch proof}:

\begin{align}
    \rgqktb (w)-\gjkbh (w)=&(\rgqktb(w)-\rgjktb(w))+(\rgjktb (w)- \gjksb (w)) \notag \\ &+(\gjksb (w) - \gjkbh (w))  \\
    \left \| \rgqktb (w)-\gjkbh (w) \right \|\stackrel{(a)}{\leq}& \left \| \rgqktb(w)-\rgjktb(w) \right \|+\left \| \rgjktb (w)- \gjksb (w) \right \| \notag \\ &+\left \| \gjksb (w) - \gjkbh (w) \right \|  \\
    \left \| \rgqktb (w)-\gjkbh (w) \right \|^{2}\stackrel{(b)}{\leq}& 3\left \| \rgqktb(w)-\rgjktb(w) \right \|^{2}+3\left \| \rgjktb (w)- \gjksb (w) \right \|^{2} \notag \\ &+3\left \| \gjksb (w) - \gjkbh (w) \right \|^{2}  \\
    \e\left \| \rgqktb (w)-\gjkbh (w) \right \|^{2}\leq& 3\e\left [\left \| \rgqktb(w)-\rgjktb(w) \right \|^{2}\right]+3\e \left[ \left \| \rgjktb (w)- \gjksb (w) \right \|^{2} \right ] \notag \\ &+3\left \| \gjksb (w) - \gjkbh (w) \right \|^{2}  \\
    \e\left \| \rgqktb (w)-\gjkbh (w) \right \|^{2}\stackrel{(c)}{\leq}& 3(C_{1}^{2}+C_{2}^{2}+C_{3}^{2}) \label{c proof connection task specific}
\end{align}

where \( (a) \) follows from triangle inequality, \( (b) \) follows from 
$(\sum_{i=1}^{3}x_{i})^{2}\leq 3(\sum_{i=1}^{3}x_{i}^{2})$ (Cauchy-Schwarz), and \( (c) \) follows from definitions of \( C_{1},C_{2},C_{3} \). Inserting \eqref{c proof connection task specific} into \eqref{c proof connection batch proof} completes the proof.

\end{proof}{}

\section{Proofs for Evolution Analysis}

\subsection{Proof of Theorem \ref{lemma network disagreement}}\label{appendix proof network disagreement lemma}

For analyzing the centroid model recursion it is useful to define the following variables which collect all variables from across the network:
\begin{align}
  \bcw_{i} &\triangleq \mathrm{col} \left \{ \w_{1,i}, \ldots, \w_{K,i} \right \} \\
  \mathcal{A} &\triangleq A \otimes I_M\\
  \widehat{\g}(\bcw_{i}) &\triangleq \mathrm{col} \left \{ \nabla \overline{Q_{1}}(\w_{1,i}), \ldots , \nabla \overline{Q_{K}}(\w_{K,i}) \right \}
\end{align}
Then, we rewrite the diffusion equations ~\eqref{eq:adapt}--\eqref{eq:combine} in a more compact form as:
\begin{equation}\label{eq:recursion}
	\bcw_i = \mathcal{A}^{\T} \left( \bcw_{i-1} - \mu \widehat{\g}(\bcw_{i-1}) \right)
\end{equation}
Multiplying this equation by \( \frac{1}{K} \mathds{1}_K^{\T} \otimes I \) from the left and using \eqref{eq:doubly-stochastic matrix properties} we get the recursion:
\begin{equation}
	\left( \frac{1}{K} \mathds{1}_K^{\T} \otimes I \right) \bcw_i = \left( \frac{1}{K} \mathds{1}_K^{\T} \otimes I \right) \bcw_{i-1} - \mu \left( \frac{1}{K} \mathds{1}_K^{\T} \otimes I \right) \widehat{\g}(\bcw_{i-1})
\end{equation}
Rewriting the centroid launch model as:
\begin{equation}
     \w_{c, i} = \sum_{k=1}^K \frac{1}{K} \w_{k, i} = \left( \frac{1}{K} \mathds{1}_K^{\T} \otimes I \right) \bcw_i 
\end{equation}
Defining the extended centroid matrix:
\begin{equation}
	\bcw_{c, i} \triangleq \mathds{1}_K \otimes \w_{c, i} = \left( \frac{1}{K} \mathds{1}_K \mathds{1}_K^{\T} \otimes I \right) \left(\bcw_{i-1} - \mu \widehat{\g}(\bcw_{i-1}) \right)
\end{equation}
It follows that:
\begin{align}
	\bcw_i - \bcw_{c, i}
	=& \left( \mathcal{A}^{\T} - \frac{1}{K} \mathds{1}_K \mathds{1}_K^{\T} \otimes I \right) \left( \bcw_{i-1} - \mu \widehat{\g}(\bcw_{i-1}) \right) \notag \\
	\stackrel{(a)}{=}& \left( \mathcal{A}^{\T} - \frac{1}{K} \mathds{1}_K \mathds{1}_K^{\T} \otimes I \right) \left( I - \frac{1}{K} \mathds{1}_K \mathds{1}_K^{\T} \otimes I  \right)  \left( \bcw_{i-1} - \mu \widehat{\g}(\bcw_{i-1}) \right) \notag \\
	=& \left( \mathcal{A}^{\T} - \frac{1}{K} \mathds{1}_K \mathds{1}_K^{\T} \otimes I \right)   \left(\bcw_{i-1} - \bcw_{c, i-1} - \mu \widehat{\g}(\bcw_{i-1}) \right)\label{eq:recursive_disagreement}
\end{align}
where \( (a) \) follows from the equality:
\begin{align}
	\left( \mathcal{A}^{\T} - \frac{1}{K} \mathds{1}_K \mathds{1}_K^{\T} \otimes I \right) \left( I - \frac{1}{K} \mathds{1}_K \mathds{1}_K^{\T} \otimes I  \right) {=} \mathcal{A}^{\T} - \frac{1}{K} \mathds{1}_K \mathds{1}_K^{\T} \otimes I
\end{align}
Taking the squared norms:
\begin{align}
	\:{\left \| \bcw_i - \bcw_{c, i} \right \|}^2 
	=&\: {\left \| \left( \mathcal{A}^{\T} - \frac{1}{K} \mathds{1}_K \mathds{1}_K^{\T} \otimes I \right)   \left(\bcw_{i-1} - \bcw_{c, i-1} - \mu \widehat{\g}(\bcw_{i-1}) \right) \right \|}^2 \notag \\
	\stackrel{(a)}{\le}&\: \lambda_2^2 {\left \| \bcw_{i-1} - \bcw_{c, i-1} - \mu \widehat{\g}(\bcw_{i-1}) \right \|}^2 \notag \\
	\stackrel{(b)}{\le}&\: \lambda_2 {\left \| \bcw_{i-1} - \bcw_{c, i-1} \right \|}^2 + \mu^2 \frac{\lambda_2^2}{1-\lambda_2} {\left \| \widehat{\g}(\bcw_{i-1}) \right \|}^2 \label{eq:disagreement_inequality}
\end{align}
where \( (a) \) follows by sub-multiplicity of \( \|\cdot\| \) and \( (b) \) follows from \( {\|x + y\|}^2 \le \frac{1}{\beta} {\|x\|}^2 + \frac{1}{1-\beta} {\|y\|}^2 \) for \( 0 < \beta < 1 \) with choice of \( \beta \):
\begin{equation}
	\lambda_2 = \left \| \mathcal{A}^{\T} - \frac{1}{K} \mathds{1}_K \mathds{1}_K^{\T} \otimes I \right \| < 1
\end{equation}
Taking expectation conditioned on \( \bcw_{i-1} \):
\begin{align}
	\e \left [ {\left \| \bcw_i - \bcw_{c, i} \right \|}^2 | \bcw_{i-1} \right ] 
	\le \:& \lambda_2 \e \left [ {\left \| \bcw_{i-1} - \bcw_{c, i-1} \right \|}^2 | \bcw_{i-1} \right ] 
	+ \mu^2 \frac{\lambda_2^2}{1-\lambda_2} \e \left [ {\left \| \widehat{\g}(\bcw_{i-1}) \right \|}^2 | \bcw_{i-1} \right ] \notag \\
	\le \:& \lambda_2 \e \left [ {\left \| \bcw_{i-1} - \bcw_{c, i-1} \right \|}^2 | \bcw_{i-1} \right ]
	+ \mu^2 \frac{\lambda_2^2}{1-\lambda_2} \sum_{k=1}^K \e \left [ {\left \| \gqkb(\w_{k, i-1}) \right \|}^2 | \bcw_{i-1} \right ] \notag \\
	\stackrel{(a)}{=} \:& \lambda_2 \e \left [ {\left \| \bcw_{i-1} - \bcw_{c, i-1} \right \|}^2 | \bcw_{i-1} \right ] + \mu^2 \frac{\lambda_2^2}{1-\lambda_2} \sum_{k=1}^K {\left \| \gjkbh(\w_{k, i-1}) \right \|}^2 \notag \\
	&+ \mu^2 \frac{\lambda_2^2}{1-\lambda_2} \sum_{k=1}^K \e \left [ {\left \| \gqkb(\w_{k, i-1})-\gjkbh(\w_{k, i-1}) \right \|}^2 | \bcw_{i-1} \right ] \notag \\
	\stackrel{(b)}{\le} \:& \lambda_2 \e \left [ {\left \| \bcw_{i-1} - \bcw_{c, i-1} \right \|}^2 | \bcw_{i-1} \right ] + \mu^2 \frac{\lambda_2^2}{1-\lambda_2} K \bbh^2
	+ \mu^2 \frac{\lambda_2^2}{1-\lambda_2} K C^2 \notag \\
	= \:& \lambda_2 \e \left [ {\left \| \bcw_{i-1} - \bcw_{c, i-1} \right \|}^2 | \bcw_{i-1} \right ] + \mu^2 \frac{\lambda_2^2}{1-\lambda_2} K \left( \bbh^2 + C^2 \right)
\end{align}
where \( (a) \) follows from dropping the cross-terms due to unbiasedness of stochastic gradient update, and \( (b) \) follows from Lemma \ref{adj objective bounded grad formula} and Lemma \ref{adj objective noise bound}.
Taking expectation again to remove the conditioning:
\begin{equation}
	\e {\left \| \bcw_i - \bcw_{c, i} \right \|}^2 \le \lambda_2 \e {\left \| \bcw_{i-1} - \bcw_{c, i-1} \right \|}^2 + \mu^2 \frac{\lambda_2^2}{1-\lambda_2} K \left( \bbh^2 + C^2 \right)
\end{equation}
We can iterate, starting from \( i = 0 \), to obtain:
\begin{align}
	\e {\left \| \bcw_i - \bcw_{c, i} \right \|}^2
	\le&\: \lambda_2^i {\|\cw_0 - \cw_{c, 0}\|}^2 + \mu^2 \frac{\lambda_2^2}{1-\lambda_2} K \left( \bbh^2 + C^2 \right) \sum_{k = 0}^{i} \lambda_2^{i-1} \notag \\
	\le&\: \lambda_2^i {\|\cw_0 - \cw_{c, 0}\|}^2 + \mu^2 \frac{\lambda_2^2}{{(1-\lambda_2)}^2} K \left( \bbh^2 + C^2 \right) \notag \\
	\stackrel{(a)}{\le}&\: \mu^2 \frac{\lambda_2^2}{{(1-\lambda_2)}^2} K \left( \bbh^2 + C^2 \right) + O(\mu^3) 
\end{align}
where \( (a) \) holds whenever:
\begin{align}
    \lambda_2^i {\|\cw_0 - \cw_{c, 0}\|}^2 &\leq c \mu^3 \Longleftrightarrow i \log \lambda_2 \le 3 \log \mu + \log c - 2 \log \|\cw_0 - \cw_{c, 0}\| \notag \\
   &\Longleftrightarrow i \ge \frac{3\log \mu }{\log \lambda_2}+O(1)=o(1/\mu)
\end{align}
where \( c \) is an arbitrary constant.

\subsection{Proof of Theorem \ref{COR:STATIONARY_POINTS}}\label{appendix stationary points corollary proof}

We first prove two intermediate lemmas, then conclude the proof.

Recall the centroid launch model:
\begin{equation}
\w_{c, i} = \sum_{k=1}^K \frac{1}{K} \w_{k, i} = \left( \frac{1}{K} \mathds{1}_K^{\T} \otimes I \right) \bcw_i 
\end{equation}
Then, we obtain the recursion:
\begin{equation}
	\w_{c, i} = \w_{c, i-1} - \frac{\mu}{K} \sum_{k=1}^K  \gqkb (\w_{k, i-1})
\end{equation}
This is almost an exact gradient descent on the aggregate cost \eqref{aggregate cost} except for the perturbation terms. Decoupling them gives us: 
\begin{align}\label{eq:perturbed_gradient_descent}
	\w_{c, i} = \w_{c, i-1} - \frac{\mu}{K} \sum_{k=1}^K \gjkbh (\w_{c, i-1}) - \mu \rd_{i-1} - \mu \s_{i}
\end{align}
where the perturbation terms are:
\begin{align}
	\rd_{i-1} &\triangleq \frac{1}{K} \sum_{k=1}^K  \left( \gjkbh (\w_{k, i-1}) - \gjkbh (\w_{c, i-1}) \right) \\
	\s_{i} &\triangleq \frac{1}{K} \sum_{k=1}^K  \left( \gqkb (\w_{k, i-1}) - \gjkbh (\w_{k, i-1}) \right)
\end{align}
Here, \( \rd_{i-1} \) measures the average disagreement with the average launch model whereas \( \s_i \) represents the average stochastic gradient noise in the process. Based on the network disagreement result Theorem \ref{lemma network disagreement}, we can bound the perturbation terms in \eqref{eq:perturbed_gradient_descent}:

\begin{lemma}[\textbf{Perturbation bounds}]\label{LEM:PERTURBATION_BOUNDS}Under assumptions \ref{assumption lipshtiz gradient}-\ref{assumption doubly stochastic matrix}, perturbation terms are bounded for sufficently small outer-step sizes \( \mu \) after sufficient number of iterations, namely:
	\begin{align}
		\e {\|\rd_{i-1}\|}^2 &\le \mu^2 \lbh^2 \frac{\lambda_2^2}{{(1-\lambda_2)}^2} \left( \bbh^2 + C^2 \right) + O(\mu^3) \\
		\e {\|\s_{i}\|}^2 &\le C^2
	\end{align}
\end{lemma}
\begin{proof}
We begin by studying more closely the perturbation term \( \s_{i-1} \) arising from the gradient approximations. We have:
\begin{align}
	\e{\| \s_{i} \|}^2 &= \e{ \left \| \frac{1}{K} \sum_{k=1}^K \left( \gqkb (\w_{k, i-1}) - \gjkbh (\w_{k, i-1}) \right) \right \|}^2 \notag \\
	&\stackrel{(a)}{\le} \frac{1}{K} \sum_{k=1}^K \e { \left \| \gqkb (\w_{k, i-1}) - \gjkbh (\w_{k, i-1}) \right \|}^2 \notag \\
	&\stackrel{(b)}{\le} \frac{1}{K} \sum_{k=1}^K  \left(  C^2 \right) \notag \\
	&= C^2 
\end{align}
where \( (a) \) follows from Jensen's inequality, and \( (b) \) follows from Lemma \ref{adj objective noise bound}. For the second perturbation term arising from the disagreement within the network, we can bound:
\begin{align}
	{\|\rd_{i-1}\|}^2 &= {\left \| \sum_{k=1}^K \frac{1}{K} \left( \gjkbh (\w_{k, i-1}) - \gjkbh (\w_{c, i-1}) \right)\right \|}^2 \notag \\
	&\stackrel{(a)}{\le}  \sum_{k=1}^N \frac{1}{K} {\left \| \gjkbh (\w_{k, i-1}) - \gjkbh (\w_{c, i-1}) \right \|}^2 \notag \\
	&\stackrel{(b)}{\le}  \frac{\lbh^2}{K} \sum_{k=1}^K {\left \| \w_{k, i-1} - \w_{c, i-1} \right \|}^2 \notag \\
	&=  \frac{\lbh^2}{K}  {\left \| \bcw_{i-1} - \bcw_{c, i-1} \right \|}^2 \label{eq:d_bound}
\end{align}
where \( (a) \) follows from Jensen's inequality, and \( (b) \) follows from Lemma \ref{adj objective lipschitz formula}. Taking the expectation and using Theorem \ref{lemma network disagreement} we complete the proof:
\begin{equation}
    \e {\|\rd_{i-1}\|}^2 \le \mu^2\lbh^2 \frac{\lambda_2^2}{{(1-\lambda_2)}^2} \left( \bbh^2 + C^2 \right) + O(\mu^3)
\end{equation}
\end{proof}
Next, we present the second lemma.
\begin{lemma}[\textbf{Descent relation}]\label{TH:DESCENT_RELATION} Under asssumptions \ref{assumption lipshtiz gradient}-\ref{assumption doubly stochastic matrix} we have the descent relation:
\begin{equation}\label{eq:descent_relation}
		\e \left [ \jbh(\w_{c, i}) | \w_{c, i-1} \right ] \le \jbh(\w_{c, i-1}) - \frac{\mu}{2} (1-2\mu\lbh) {\left \| \gjbh(\w_{c, i-1})\right\|}^2 + \frac{1}{2}\mu^2\lbh C^2+ O(\mu^3)
\end{equation}
\end{lemma}
\begin{proof}
First, observe that since each individual \( \jkbh(\cdot) \) has Lipschitz gradients by Lemma \ref{adj objective lipschitz formula}, the same holds for the average:
\begin{align}
	\:\left \| \gjbh(w) - \gjbh(u) \right \| 
	=&\:\left \| \nabla \left( \sum_{k=1}^K \frac{1}{K} \jkbh(w) \right) - \nabla \left( \sum_{k=1}^K \frac{1}{K} \jkbh(u) \right) \right \| \notag \\
	=&\: \left \| \sum_{k=1}^K \frac{1}{K} \left( \gjkbh(w) - \gjkbh(u) \right) \right \| \notag \\
	\stackrel{(a)}{\le}&\: \sum_{k=1}^N \frac{1}{K} \left \| \gjkbh(w) - \gjkbh(u) \right \| \notag \\
	\stackrel{(b)}{\le}&\: \sum_{k=1}^N \frac{1}{K} \lbh \left \| w - u \right \| \notag \\
	=&\: \lbh \left \| w - u \right \|
\end{align}
where \( (a) \) follows from Jensen's inequality, \( (b) \) follows from Lemma \ref{adj objective lipschitz formula}.
This property then implies the following bound:
\begin{align}
	\jbh(\w_{c, i}) \le&\: \jbh(\w_{c, i-1}) + {\gjbh(\w_{c, i-1})}^{\T} \left( \w_{c, i} - \w_{c, i-1} \right) \:+ \frac{\lbh}{2} {\left \| \w_{c, i} - \w_{c, i-1} \right \|}^2 \notag \\
	 \stackrel{(a)}{\le}&\: \jbh(\w_{c, i-1}) - \mu {\left \| {\gjbh(\w_{c, i-1})} \right \|}^2 
	\:- \mu {\gjbh(\w_{c, i-1})}^{\T} \left( \rd_{i-1} + \s_i \right) 
	\:+ \mu^2 \frac{\lbh}{2} {\left \| {\gjbh(\w_{c, i-1})} + \rd_{i-1} + \s_i \right \|}^2
\end{align}

where \( (a) \) follows from \eqref{eq:perturbed_gradient_descent}.
Taking expectations, conditioned on \( \bcw_{i-1} \) yields:
\begin{align}
	\:\e \left [ \jbh(\w_{c, i}) | \bcw_{i-1} \right ] 
	\stackrel{(a)}{\le}&\: \jbh(\w_{c, i-1}) - \mu {\left \| {\gjbh(\w_{c, i-1})} \right \|}^2 - \mu {\gjbh(\w_{c, i-1})}^{\T} \rd_{i-1} \notag \\
	&\:+ \mu^2 \frac{\lbh}{2} {\left \| {\gjbh(\w_{c, i-1})} + \rd_{i-1} \right \|}^2 + \mu^2 \frac{\lbh}{2} \e \left[{\|\s_i\|}^2 | \bcw_{i-1} \right ] \notag \\
	\stackrel{(b)}{\le}&\: \jbh(\w_{c, i-1}) - \mu {\left \| {\gjbh(\w_{c, i-1})} \right \|}^2 + \frac{\mu}{2} {\left \| \gjbh(\w_{c, i-1})\right\|}^{2} + \frac{\mu}{2} {\left \| \rd_{i-1} \right \|}^2 \notag \\
	&\:+ \mu^2 \lbh {\left \| {\gjbh(\w_{c, i-1})} \right \|}^2 + \mu^2 \lbh^2 {\left \| \rd_{i-1} \right \|}^2 + \mu^2 \frac{\lbh}{2} \e \left[{\|\s_i\|}^2 | \bcw_{i-1} \right ] \notag \\
	\le&\: \jbh(\w_{c, i-1}) - \frac{\mu}{2}\left(1 - 2 \mu \lbh \right) {\left \| {\gjbh(\w_{c, i-1})} \right \|}^2 + \frac{\mu}{2} \left( 1 + 2 \mu \lbh \right) {\left \| \rd_{i-1} \right \|}^2 \notag \\
	&\: + \mu^2 \frac{\lbh}{2} \e \left[{\|\s_i\|}^2 | \bcw_{i-1} \right ]
\end{align}
where \( (a) \) follows from \( \e \s_i = 0 \), and \( (b) \) follows from Cauchy-Schwarz and \( ab\leq \frac{a^{2}+b^{2}}{2}\) .

Taking expectations to remove the conditioning and the bounds from Lemma~\ref{LEM:PERTURBATION_BOUNDS} yields:
\begin{align}
	\e \jbh(\w_{c, i}) \le& \e \jbh(\w_{c, i-1}) - \frac{\mu}{2}\left(1 - 2 \mu \lbh \right) \e {\left \| {\gjbh(\w_{c, i-1})} \right \|}^2 
	\:+ \frac{\mu^3}{2} \left( 1 + 2 \mu \lbh \right) \lbh^2  \frac{\lambda_2^2}{{(1-\lambda_2)}^2} \left( \bbh^2 + C^2 \right) \notag \\
	&\:+ \mu^2 \frac{\lbh}{2} C^2 + O(\mu^4) \notag \\
	=& \e \jbh(\w_{c, i-1}) - \frac{\mu}{2}\left(1 - 2 \mu \lbh \right) \e {\left \| {\gjbh(\w_{c, i-1})} \right \|}^2 
	\:+ \mu^2 \frac{\lbh}{2}  C^2  + O(\mu^3)
\end{align}
\end{proof}

The proof of the theorem is based on contradiction. First define:
\begin{align}
    c_1 &\triangleq \frac{1-2\mu\lbh}{2} \\
    c_2 &\triangleq \frac{\lbh C^2}{2}+O(\mu) 
\end{align}
We will prove 
\begin{equation}
	\e {\left \|\nabla \jbh(\w_{c, i^{\star}})\right\|}^2 \le 2 \mu \frac{c_2}{c_1}
\end{equation}
\begin{equation}
    i^{\star} \le \left(\frac{\jbh(w_0) - \jbh^o}{c_2}\right) 1/\mu^2
\end{equation}
which correspond to \eqref{eq:stationary_points} and \eqref{eq:istar number of iterations}, respectively.
Descent relation \eqref{eq:descent_relation} can be rewritten as:
\begin{equation}\label{eq:descent_relation with deltas}
	\e \left [ \jbh(\w_{c, i}) | \w_{c, i-1} \right ] \le \jbh(\w_{c, i-1}) - \mu c_1 {\left \| \gjbh(\w_{c, i-1})\right\|}^2 + \mu^2 c_2
\end{equation}
Suppose there is no time instant \( i^{\star} \) satisfying \( {\left \|\gjbh(\w_{c, i^{\star}})\right\|}^2 \le 2 \mu \frac{c_2}{c_1} \). Then for any time \( i \) we obtain:
\begin{align}
	\e \jbh(\w_{c, i}) &\stackrel{(a)}{\le} \jbh(w_0) - \mu c_1 \sum_{k=1}^i \left( \e{\left \|\gjbh(\w_{c, k-1})\right\|}^2 - \mu \frac{c_2}{c_1} \right) \notag \\
	&\le \jbh(w_0) - \mu^2 c_2 i
\end{align}
where \( (a) \) follows from starting from the first iteration and iterating over \eqref{eq:descent_relation with deltas}. But when the limit is taken \( \lim_{i \to \infty} \e \jbh(\w_{c, i}) \le - \infty \) it contradicts the boundedness from below assumption \( \jbh(w) \ge \jbh^o \) for all \( w \). This proves \eqref{eq:stationary_points}. In order to prove \eqref{eq:istar number of iterations}, iterate over ~\eqref{eq:descent_relation with deltas} up to time \( i^{\star} \), the first time instant where \( \e {\left \|\nabla \jbh(\w_{c, i^{\star}})\right\|}^2 \le 2 \mu \frac{c_2}{c_1}\) holds:
\begin{align}
	\jbh^o &\le \e \jbh(\w_{c, i^{\star}}) \notag \\
	&\le \jbh(w_0) - \mu c_1 \sum_{k=1}^{i^{\star}} \left( \e{\left \|\gjbh(\w_{c, k-1})\right\|}^2 - \mu \frac{c_2}{c_1} \right) \notag \\
	&\le \jbh(w_0) - \mu^2 c_2 i^{\star}
\end{align}
Rearranging completes the proof.

\subsection{Proof of Corollary \ref{COR:STATIONARY_POINTS_MAML_OBJ)}}\label{appendix maml stationary points corollary proof}

We begin by adding and subtracting \(  {\left \|\gjbh (\w_{c, i^{\star}})\right\|}^2 \):
\begin{align}
    \e {\left \|\gjb (\w_{c, i^{\star}})\right\|}^2 &= \e {\left \|\gjb (\w_{c, i^{\star}}) - \gjbh (\w_{c, i^{\star}}) + \gjbh (\w_{c, i^{\star}}) \right\|}^2 \notag \\
    &\stackrel{(a)}{\leq}2\e  {\left \|\gjb (\w_{c, i^{\star}}) - \gjbh (\w_{c, i^{\star}})\right\|}^2+2 \e {\left \| \gjbh (\w_{c, i^{\star}}) \right\|}^2
\end{align}
where \( (a) \) follows from the inequality \( \norm{a+b}^2 \leq 2\norm{a}^2 + 2\norm{b}^2 \). Inserting \eqref{eq:pert_analysis_grad} and \eqref{eq:stationary_points} completes the proof.

\section{Additional Experiment Details}

\subsection{The Regression Experiment Details}\label{appendix regression details}

The same model architecture (a neural network with 2 hidden layers of 40 neurons with ReLu activations) is used for each agent. The loss function is the mean-squared error. As in \citep{FinnAL17}, while training, 10 random points (10-shot) are chosen from each sinusoid and used with 1 stochastic gradient update (\(\al=0.01 \)). For the Adam experiment \( \mu = 0.001 \) and for the SGD experiment \( \mu=0.005 \). Each agent is trained on 4000 tasks over 6 epochs (total number of iterations = 24000). As in training, 10 data points from each sinusoid with 1 gradient update is used for adaptation.

\subsection{The Classification Experiment Dataset Details}\label{appendix dataset details}
The Omniglot dataset comprises 1623 characters from 50 different alphabets. Each character has 20 samples, which were hand drawn by 20 different people. Therefore, it is suitable for few-shot learning scenarios as there is small number of data per class.

\begin{figure}[H]
\centering
\subfloat[][]{
  \includegraphics[width=.1\linewidth]{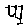}%
}
\hfil
\subfloat[][]{
  \includegraphics[width=.1\linewidth]{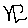}%
}
\hfil
\subfloat[][]{
  \includegraphics[width=.1\linewidth]{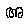}
}
\hfil
\subfloat[][]{
  \includegraphics[width=.1\linewidth]{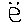}
}
\hfil
\subfloat[][]{
  \includegraphics[width=.1\linewidth]{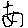}
}
\hfil
\subfloat[][]{
  \includegraphics[width=.1\linewidth]{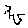}
}
\caption{The Omniglot dataset: Samples from 6 different characters}\label{omniglot samples}
\end{figure}

The MiniImagenet dataset consists of 100 classes from ImageNet \citep{imagenet} with 600 samples from each class. It captures the complexity of ImageNet samples while not working on the full dataset which is huge.

\begin{figure}[H]
\centering
\subfloat[][]{
  \includegraphics[width=.25\linewidth]{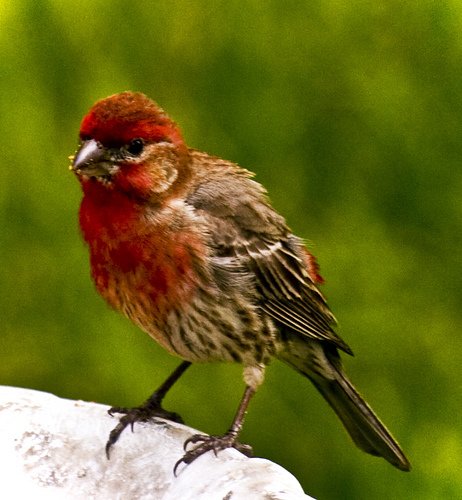}%
}
\hfil
\subfloat[][]{
  \includegraphics[width=.25\linewidth]{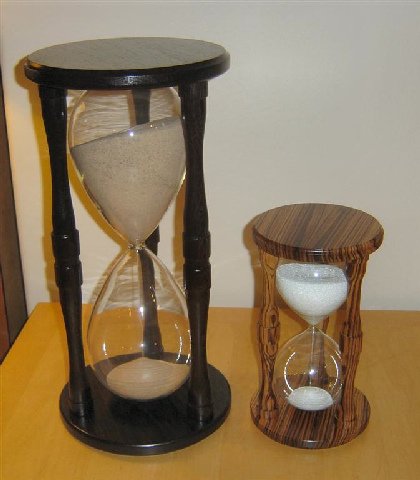}%
}
\hfil
\subfloat[][]{
  \includegraphics[width=.25\linewidth]{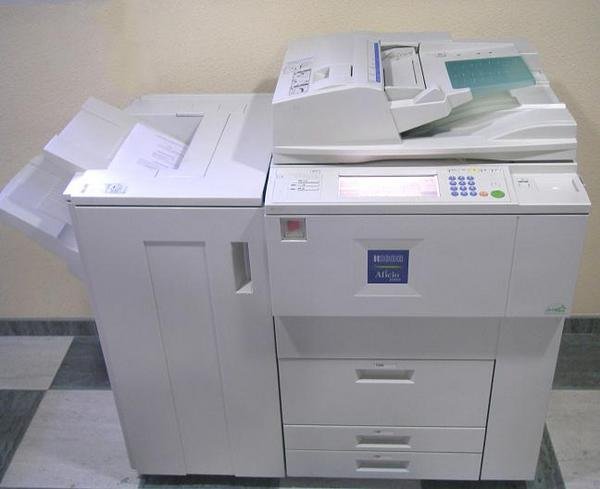}
}
\caption{The MiniImagenet dataset: Samples from 3 different classes}\label{omniglot samples}
\end{figure}

\subsection{The Classification Experiment Details}\label{appendix classification details}
Following \citep{Santoro16} and \citep{FinnAL17}, Omniglot is augmented with multiples of 90 degree rotations of the images. All agents are equipped with the same convolutional neural network architecture. Convolutional neural network architectures are the same as the architectures in \citep{FinnAL17} which are based on \citep{VinyalsMatchingNetworks}. The only difference is that we use max-pooling instead of strided convolutions for Omniglot. 

In all simulations, each agent runs over 1000 batches of tasks over 3 epochs. For the Adam experiments \( \mu = 0.001 \) and for the SGD experiments \( \mu=0.1 \). For Omniglot experiments, single gradient step is used for adaptation in both training and testing and \( \al=0.4 \). Training meta-batch size is equal to 16 for 5-way 1-shot and 8 for 5-way 5-shot. The plots are showing an average result of 100 tasks as testing meta-batch consists of 100 tasks. For MiniImagenet experiments, 10-query examples are used, testing meta-batch consists of 25 tasks and \( \al=0.01\). The number of gradient updates is equal to 5 for training, 10 for testing.  For 5-way 1-shot, training meta-batch has 4 tasks whereas 5-way 5-shot training meta-batch has 2 tasks. Note that the first testings occur after the first training step. In other words, the first data of all classification plots are at 1st iteration, not at 0th iteration.
\subsection{Additional Plots}\label{appendix additional plots}
In this section, we provide additional plots.

In Figure \ref{fig:reg_sgd}, the results of the SGD experiment on regression setting can be found. Evidently, Dif-MAML is matching the centralized solution and outperforming the non-cooperative solution as our analysis suggested. Also, similar to the Adam experiment, the relative performances stay the same with the number of gradient updates.

In Figures \ref{fig:min_sgd_5w5sr},\ref{fig:min_5w1s},\ref{fig:omn_5w1s},\ref{fig:omn_20w1s} additional plots for MiniImagenet 5-way 5-shot, MiniImagenet 5-way 1-shot, Omniglot 5-way 1-shot and Omniglot 20-way 1-shot can be found, respectively. The results confirm that our conclusions are valid for different task distributions, and they extend to Adam as well as multi-step adaptation in the inner loop.

\begin{figure}[H]
\centering
\subfloat[][]{
  \includegraphics[width=.45\linewidth]{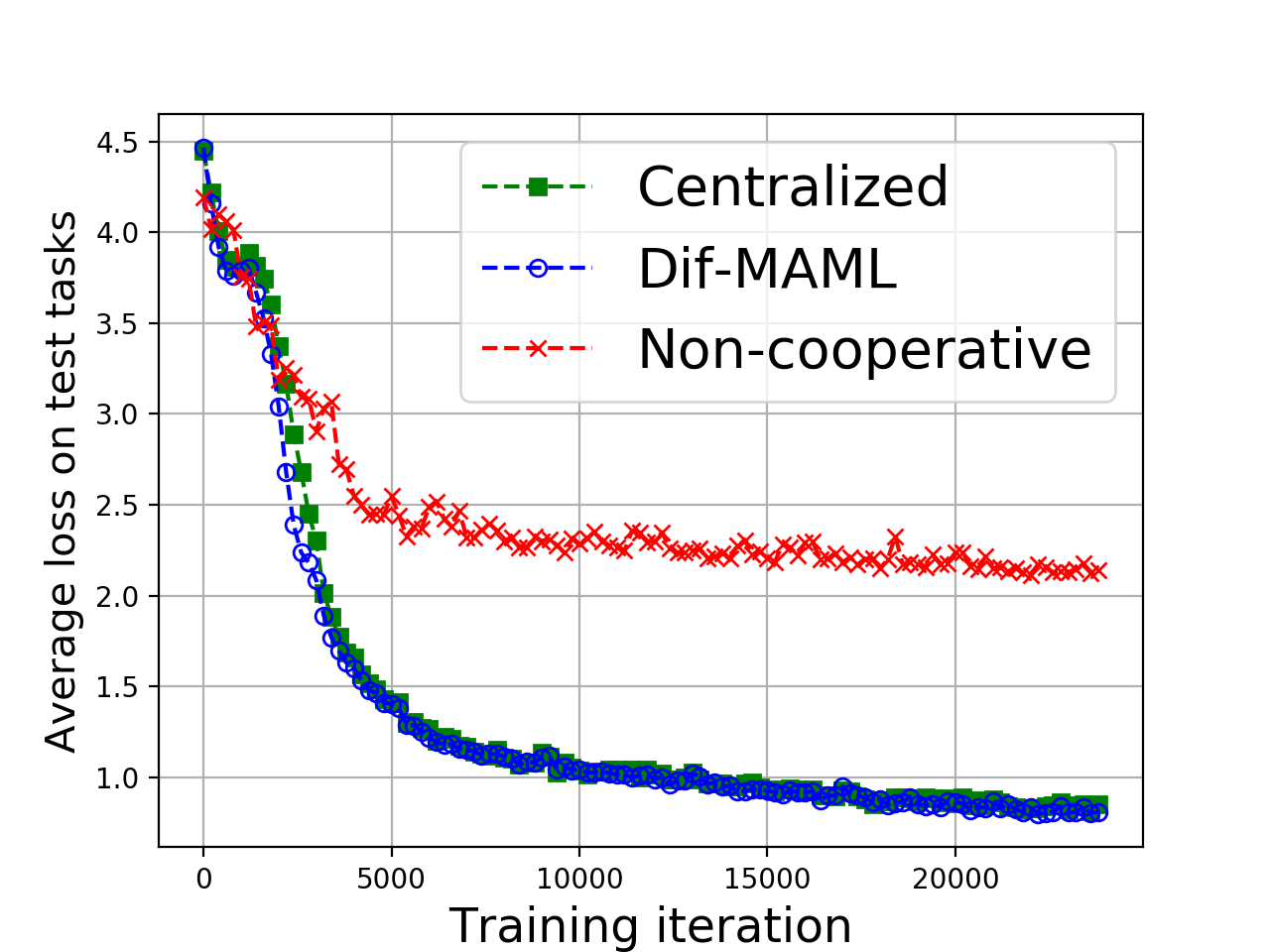}\label{fig:reg_sgd_train}%
}
\hfil
\subfloat[][]{
  \includegraphics[width=.45\linewidth]{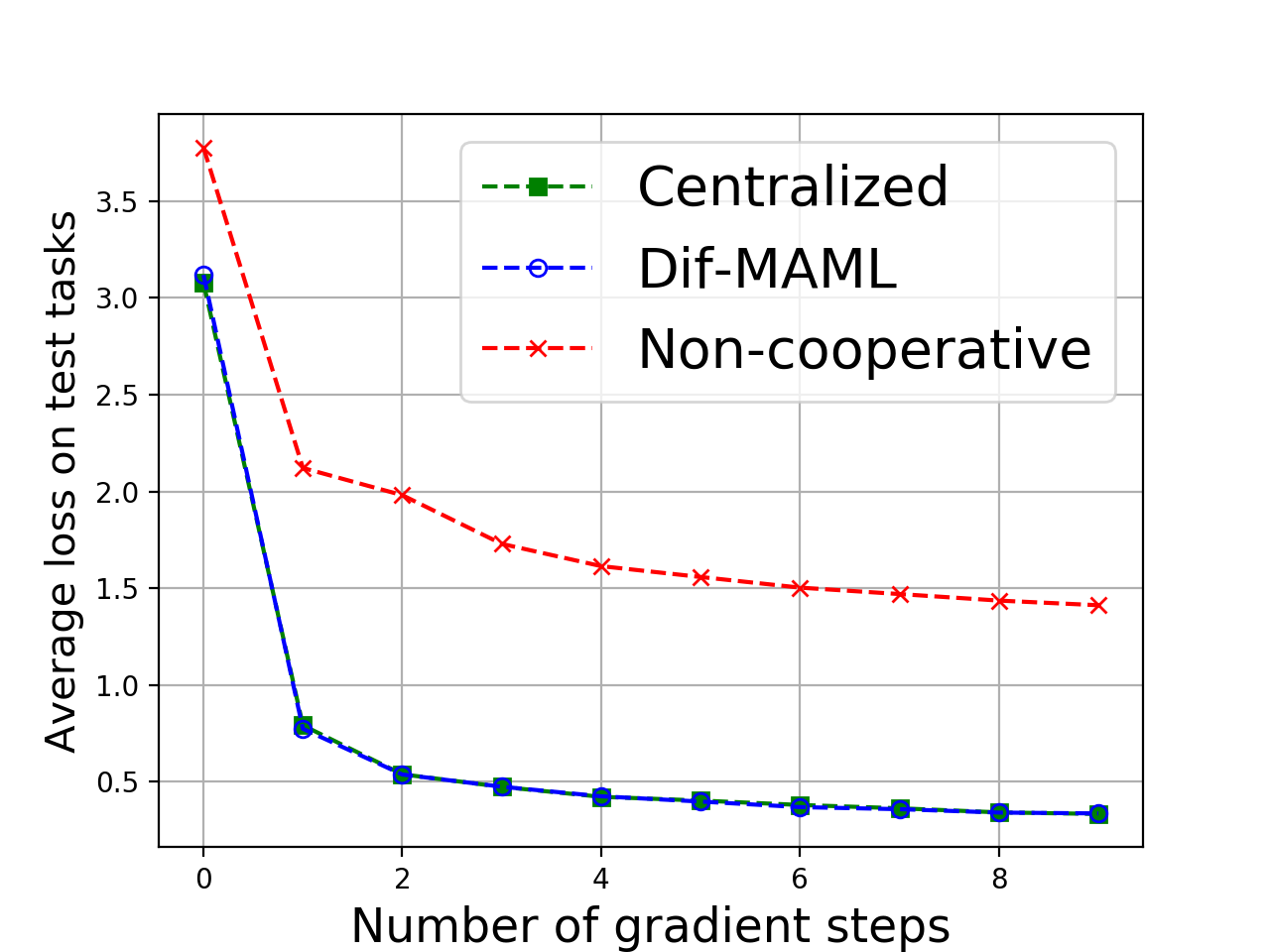}\label{fig:reg_sgd_test}%
}
\caption{(a) Regression-test losses during training -SGD (b) Regression- test losses with respect to number of gradient steps after training-SGD}\label{fig:reg_sgd}
\end{figure}

\begin{figure}[H]
\centering
  \includegraphics[width=.45\linewidth]{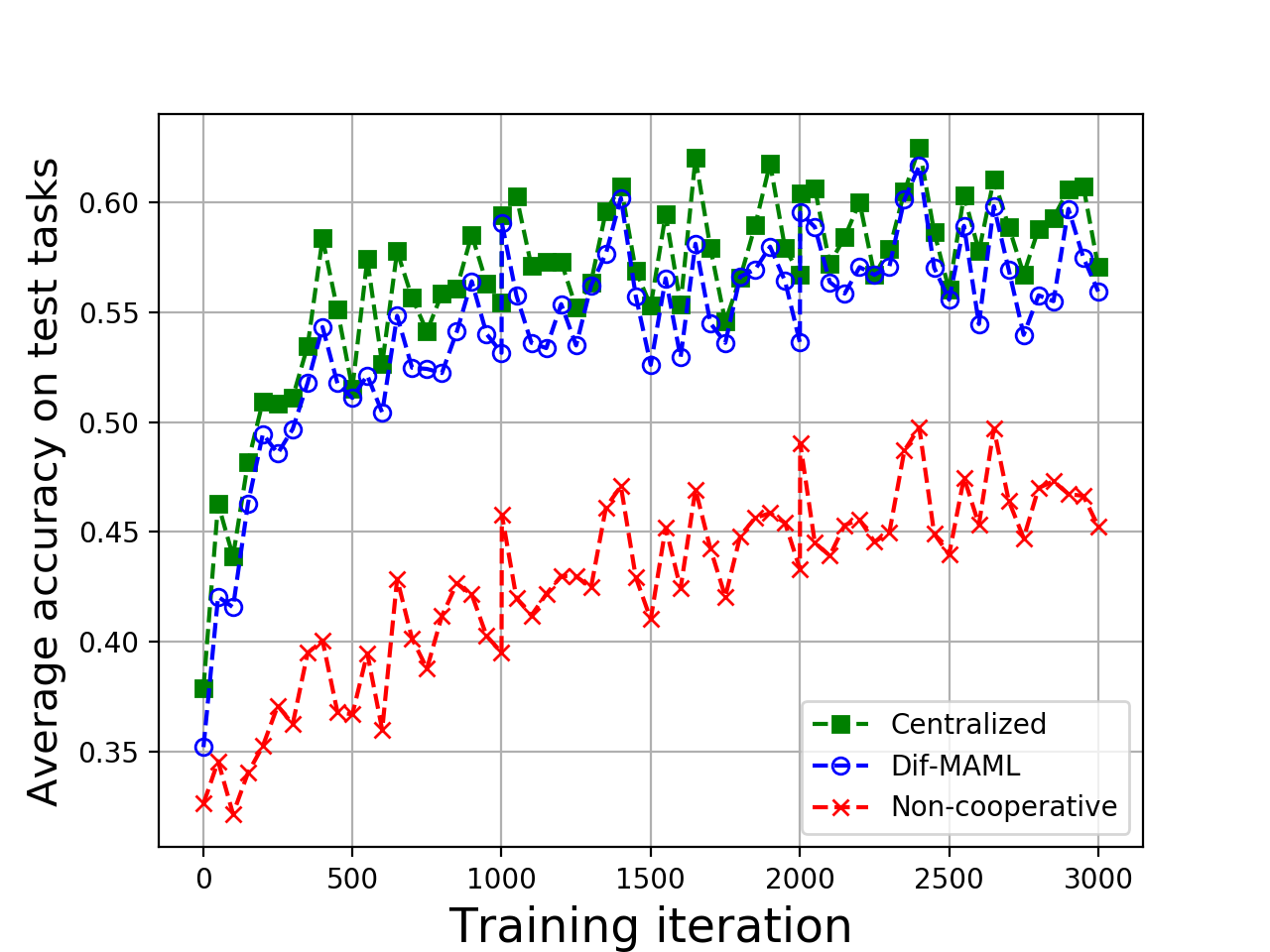}\label{fig:min_sgd_5w5s}%
\caption{MiniImagenet test accuracies during training process: 5-way 5-shot SGD}\label{fig:min_sgd_5w5sr}
\end{figure}

\begin{figure}[H]
\centering
\subfloat[][]{
  \includegraphics[width=.45\linewidth]{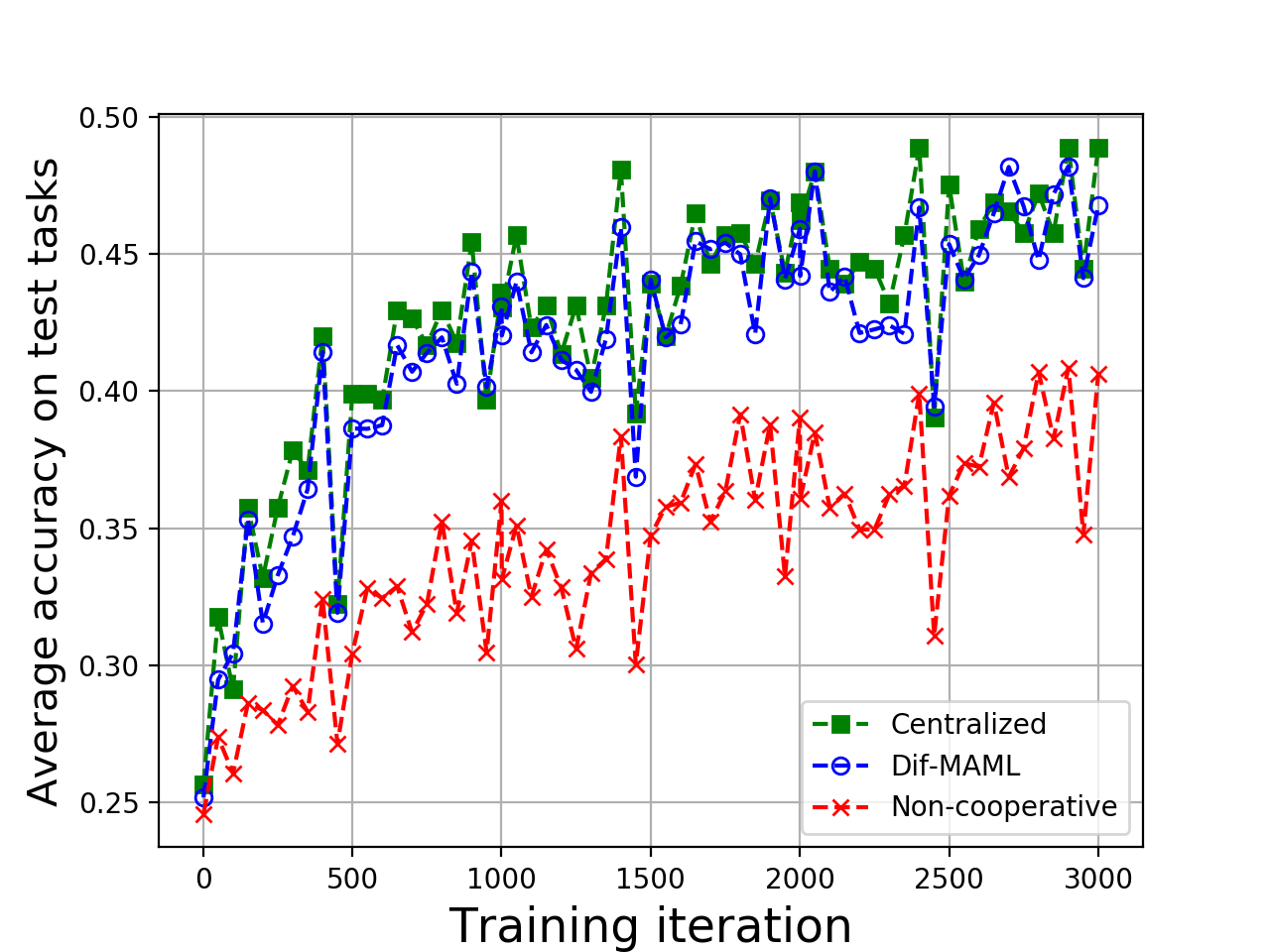}\label{fig:min_adam_5w1s}%
}
\hfil
\subfloat[][]{
  \includegraphics[width=.45\linewidth]{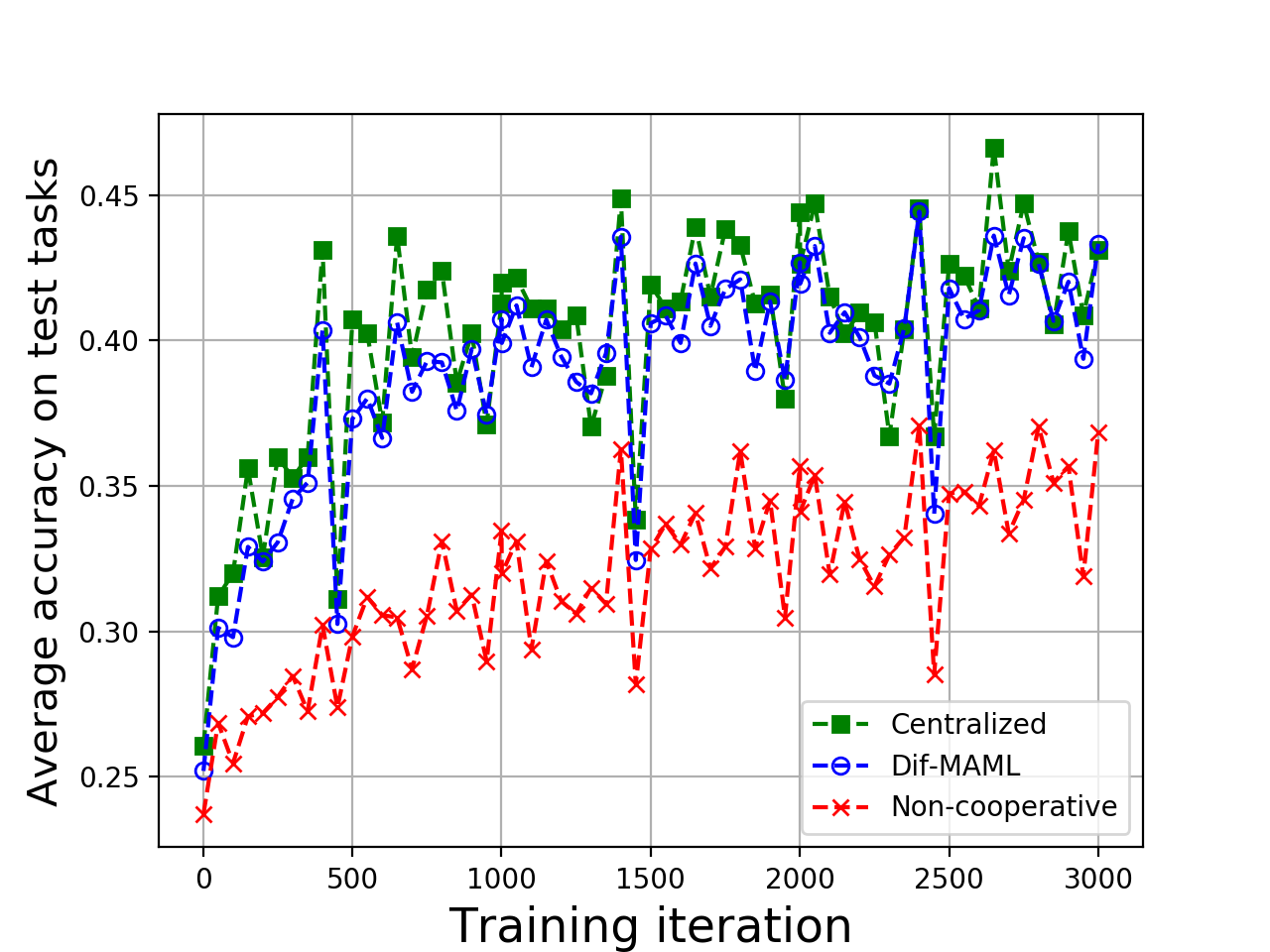}\label{fig:min_sgd_5w1s}%
}
\caption{MiniImagenet test accuracies during training process 5-way 1-shot (a) Adam (b) SGD}\label{fig:min_5w1s}
\end{figure}

\begin{figure}[H]
\centering
\subfloat[][]{
  \includegraphics[width=.45\linewidth]{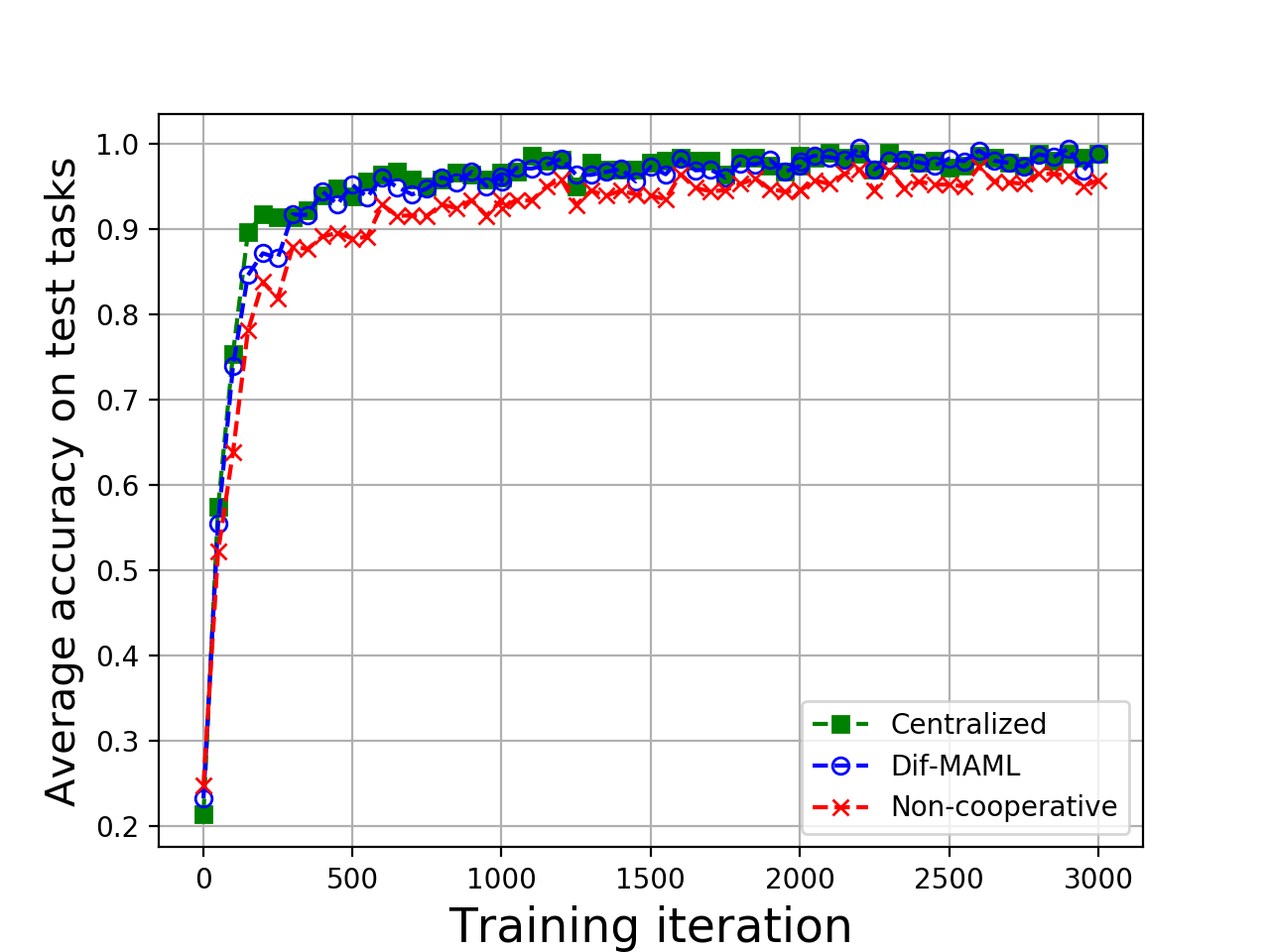}\label{fig:omn_adam_5w1s}%
}
\hfil
\subfloat[][]{
  \includegraphics[width=.45\linewidth]{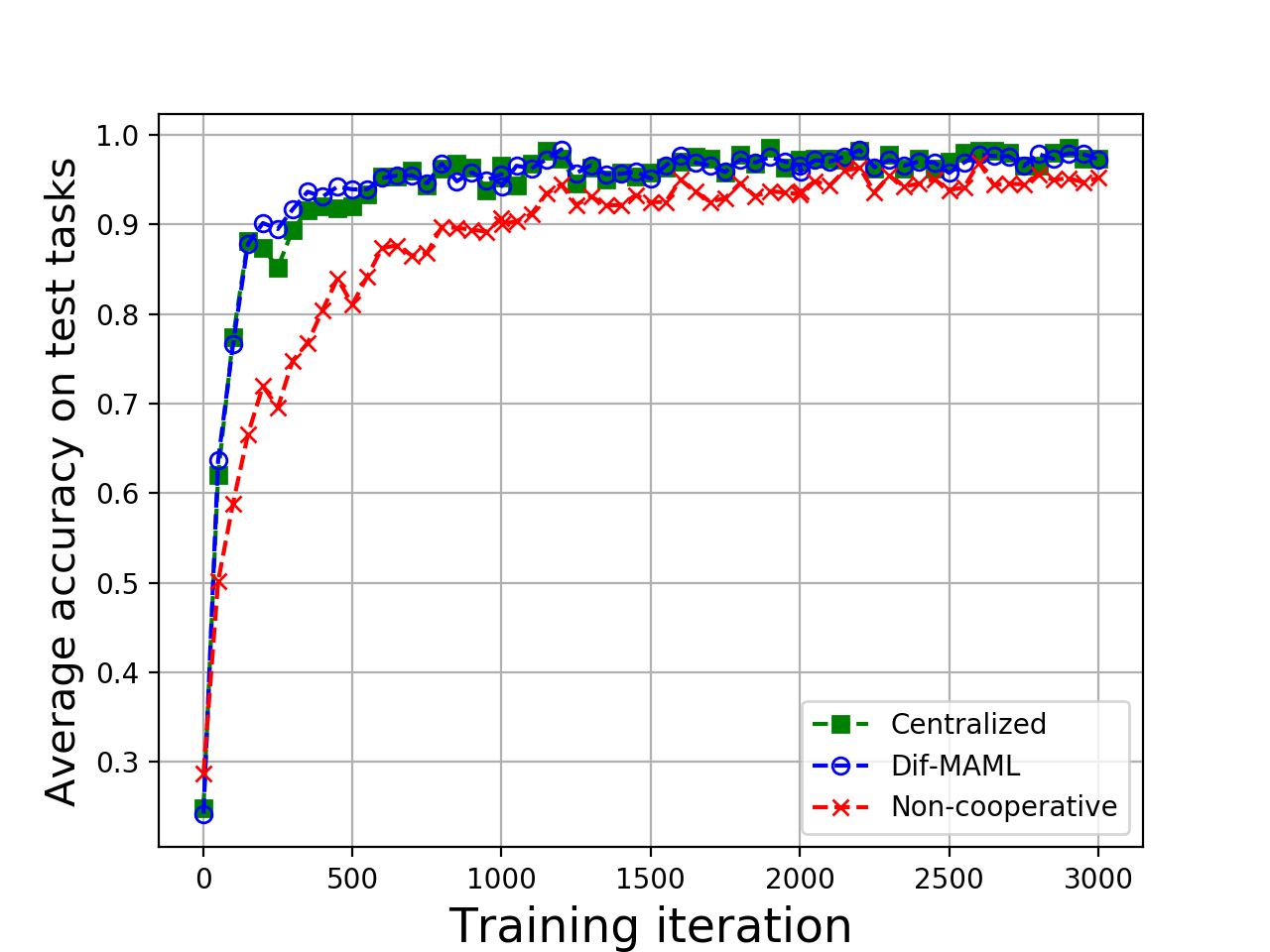}\label{fig:omn_sgd_5w1s}%
}
\caption{Omniglot test accuracies during training process 5-way 1-shot (a) Adam (b) SGD}\label{fig:omn_5w1s}
\end{figure}

\begin{figure}[H]
\centering
\subfloat[][]{
  \includegraphics[width=.45\linewidth]{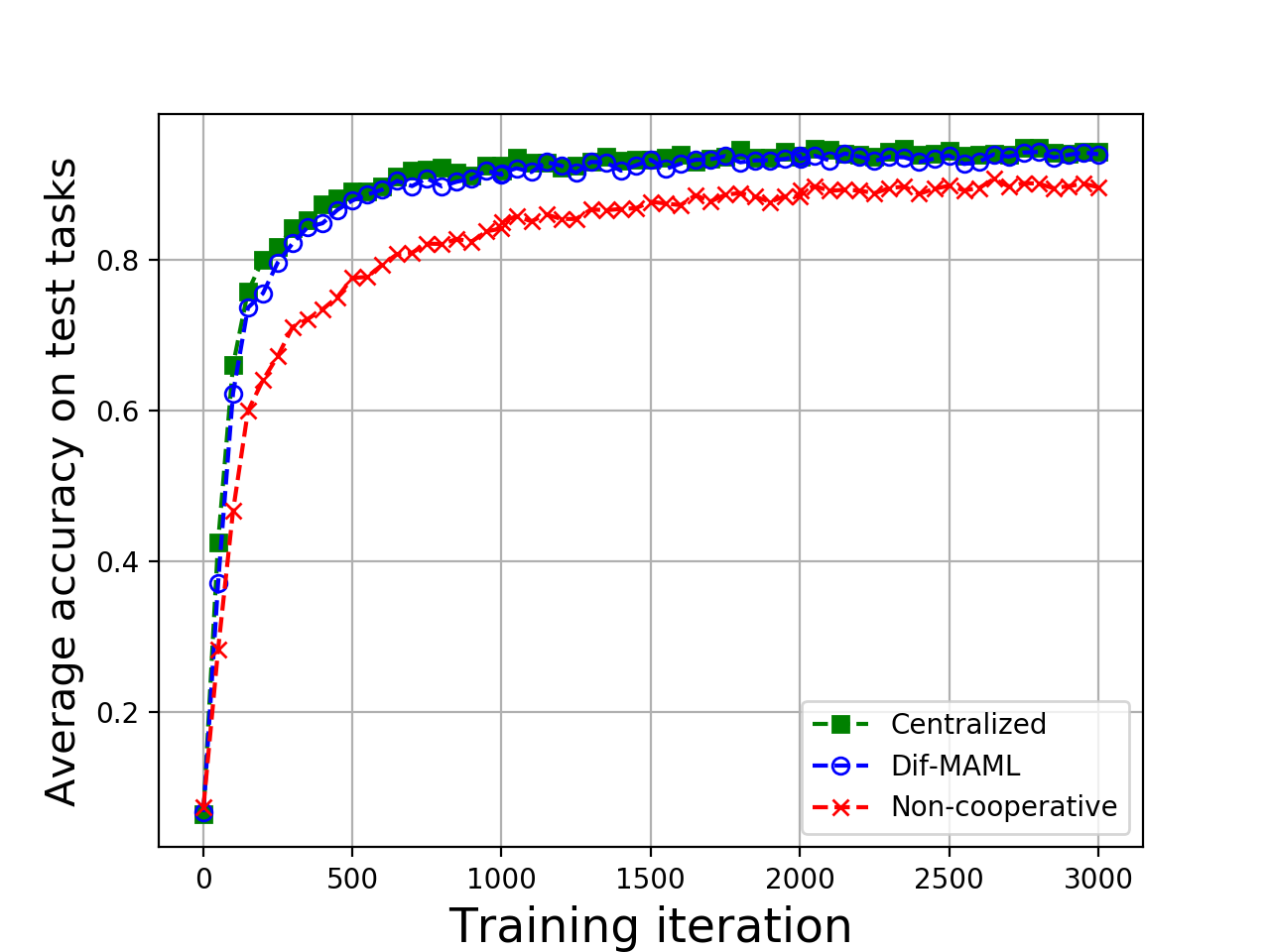}\label{fig:omn_adam_20w1s}%
}
\hfil
\subfloat[][]{
  \includegraphics[width=.45\linewidth]{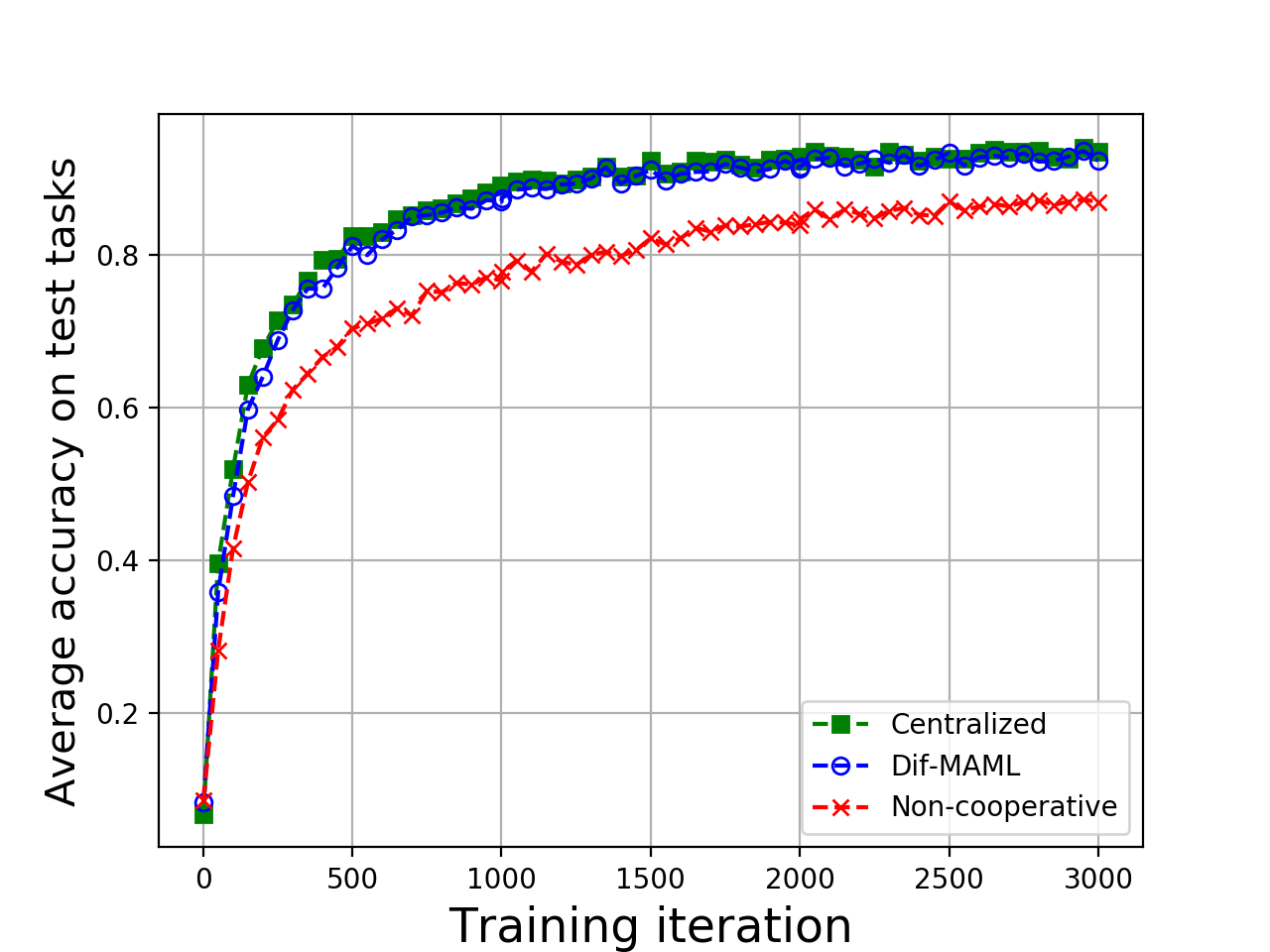}\label{fig:omn_sgd_20w1s}%
}
\caption{Omniglot test accuracies during training process 20-way 1-shot (a) Adam (b) SGD}\label{fig:omn_20w1s}
\end{figure}

\end{document}

%% file: custom/macros.tex
\DeclarePairedDelimiter\abs{\lvert}{\rvert}%
\DeclarePairedDelimiter\norm{\lVert}{\rVert}%

\newtheorem{assumption}{Assumption}
\newtheorem{theorem}{Theorem}
\newtheorem{corollary}{Corollary}
\newtheorem{lemma}{Lemma}

\DeclareMathOperator{\cw}{{\scriptstyle\mathcal{W}}}
\DeclareMathOperator{\bcw}{{\boldsymbol{\scriptstyle\mathcal{W}}}}
\DeclareMathOperator*{\argmin}{arg\,min}
\DeclareMathOperator{\T}{\mathsf{T}}

\newcommand{\jb}{\overline{J}}
\newcommand{\jbh}{\widehat{J}}

\newcommand{\bbh}{\widehat{B}}
\newcommand{\lbh}{\widehat{L}}

\newcommand{\jkb}{\overline{J_{k}}}
\newcommand{\qkb}{\overline{Q_{k}}}
\newcommand{\jkbh}{\widehat{J_{k}}}
\newcommand{\gjkbh}{\nabla \jkbh}
\newcommand{\gjkb}{\nabla \jkb}
\newcommand{\gqkb}{\nabla \qkb}
\newcommand{\gjbh}{\nabla \jbh}
\newcommand{\gjb}{\nabla \jb}

\newcommand{\rjktb}{\jkb^{(\boldsymbol{t})}}
\newcommand{\rqktb}{\qkb^{(\boldsymbol{t})}}
\newcommand{\rgjktb}{\nabla \rjktb}
\newcommand{\rgqktb}{\nabla \rqktb}

\newcommand{\jkt}{J_{k}^{(t)}}
\newcommand{\qkt}{Q_{k}^{(t)}}
\newcommand{\gjkt}{\nabla \jkt}
\newcommand{\gqkt}{\nabla \qkt}
\newcommand{\hjkt}{\nabla^{2} \jkt}
\newcommand{\hqkt}{\nabla^{2} \qkt}

\newcommand{\rjkt}{J_{k}^{(\boldsymbol{t})}}
\newcommand{\rqkt}{Q_{k}^{(\boldsymbol{t})}}
\newcommand{\rgjkt}{\nabla \rjkt}
\newcommand{\rgqkt}{\nabla \rqkt}
\newcommand{\rhjkt}{\nabla^{2} \rjkt}
\newcommand{\rhqkt}{\nabla^{2} \rqkt}

\newcommand{\jk}{J_{k}}
\newcommand{\gjk}{\nabla \jk}
\newcommand{\hjk}{\nabla^{2} \jk}

\newcommand{\jksb}{\overline{\jk^{*}}}
\newcommand{\gjksb}{\nabla \jksb}

\newcommand{\e}{\mathbb{E}}
\newcommand{\etk}{\mathbb{E}_{t \sim \pi_{k}}}

\newcommand{\ds}{\mathcal{X}}

\newcommand{\xin}{\ds_{in}^{(t)}}

\newcommand{\xout}{\ds_{o}^{(t)}}

\newcommand{\rds}{\boldsymbol{\mathcal{X}}}

\newcommand{\rxin}{\rds_{in}^{(t)}}

\newcommand{\rxout}{\rds_{o}^{(t)}}

\newcommand{\x}{\bm{x}}
\newcommand{\xinsize}{\abs{\ds_{in}}}
\newcommand{\xosize}{\abs{\ds_{o}}}

\newcommand{\eds}{\e_{\xin , \xout}}
\newcommand{\exin}{\e_{\xin}}
\newcommand{\exout}{\e_{\xout}}

\newcommand{\al}{\alpha}
\newcommand{\sigh}{\sigma_{H}}
\newcommand{\sigg}{\sigma_{G}}
\newcommand{\gamg}{\gamma_{G}}
\newcommand{\gamh}{\gamma_{H}}
\newcommand{\averagetasks}{\frac{1}{\abs {\mathcal{S}_{k}}} \sum_{\boldsymbol{t} \in \boldsymbol{\mathcal{S}}_{k}}}

\newcommand{\erhxt}{\boldsymbol{e}_{h,x}^{(t)}}
\newcommand{\ergot}{\boldsymbol{e}_{g,o}^{(t)}}
\newcommand{\ergxt}{\boldsymbol{e}_{g,x}^{(t)}}
\newcommand{\erhtt}{\boldsymbol{e}_{h,t}^{(t)}}
\newcommand{\ergtt}{\boldsymbol{e}_{g,t}^{(t)}}

\newcommand{\w}{\boldsymbol{w}}
\newcommand{\g}{\boldsymbol{g}}
\newcommand{\s}{\boldsymbol{s}}
\newcommand{\rd}{\boldsymbol{d}}
\newcommand{\rr}{\boldsymbol{r}}

\newcommand{\nkt}{N_{k}^{(t)}}

%% file: multi_maml.bbl
\begin{thebibliography}{}

\bibitem[Altae-Tran et~al., 2017]{drugdiscovery}
Altae-Tran, H., Ramsundar, B., Pappu, A.~S., and Pande, V. (2017).
\newblock Low data drug discovery with one-shot learning.
\newblock {\em ACS central science}, 3(4):283--293.

\bibitem[Andrychowicz et~al., 2016]{Andrychowicz16}
Andrychowicz, M., Denil, M., G\'{o}mez, S., Hoffman, M.~W., Pfau, D., Schaul,
  T., Shillingford, B., and de~Freitas, N. (2016).
\newblock Learning to learn by gradient descent by gradient descent.
\newblock In {\em Advances in Neural Information Processing Systems 29}, pages
  3981--3989. Barcelona, Spain.

\bibitem[Balcan et~al., 2019]{provablebalcan}
Balcan, M.-F., Khodak, M., and Talwalkar, A. (2019).
\newblock Provable guarantees for gradient-based meta-learning.
\newblock In {\em Proc. 36th International Conference on Machine Learning},
  volume~97, pages 424--433, Long Beach, California, USA.

\bibitem[Bengio et~al., 1992]{bengio92}
Bengio, S., Bengio, Y., Cloutier, J., and Gecsei, J. (1992).
\newblock On the optimization of a synaptic learning rule.
\newblock In {\em Preprints Conf. Optimality in Artificial and Biological
  Neural Networks}.

\bibitem[{Bengio} et~al., 1991]{Bengio91}
{Bengio}, Y., {Bengio}, S., and {Cloutier}, J. (1991).
\newblock Learning a synaptic learning rule.
\newblock In {\em Proc. International Joint Conference on Neural Networks},
  volume~ii, page 969.

\bibitem[Beni, 2004]{Beni04}
Beni, G. (2004).
\newblock From swarm intelligence to swarm robotics.
\newblock In {\em Proc. International Workshop on Swarm Robotics}, pages 1--9,
  Santa Monica, CA, USA.

\bibitem[Chen et~al., 2018]{FederatedML}
Chen, F., Luo, M., Dong, Z., Li, Z., and He, X. (2018).
\newblock Federated meta-learning with fast convergence and efficient
  communication.
\newblock {\em {Available as} arXiv: 1802.07876}.

\bibitem[Chen and Sayed, 2015a]{Chen15transient}
Chen, J. and Sayed, A.~H. (2015a).
\newblock On the learning behavior of adaptive networks - {Part I}: Transient
  analysis.
\newblock {\em IEEE Transactions on Information Theory}, 61(6):3487--3517.

\bibitem[Chen and Sayed, 2015b]{Chen15performance}
Chen, J. and Sayed, A.~H. (2015b).
\newblock On the learning behavior of adaptive networks - {Part II}:
  Performance analysis.
\newblock {\em IEEE Transactions on Information Theory}, 61(6):3518--3548.

\bibitem[{Deng} et~al., 2009]{imagenet}
{Deng}, J., {Dong}, W., {Socher}, R., {Li}, L., {Kai Li}, and {Li Fei-Fei}
  (2009).
\newblock Imagenet: A large-scale hierarchical image database.
\newblock In {\em Proc. IEEE Conference on Computer Vision and Pattern
  Recognition}, pages 248--255.

\bibitem[{Dunbabin} and {Marques}, 2012]{RobotsEnvMonitor}
{Dunbabin}, M. and {Marques}, L. (2012).
\newblock Robots for environmental monitoring: Significant advancements and
  applications.
\newblock {\em IEEE Robotics Automation Magazine}, 19(1):24--39.

\bibitem[Fallah et~al., 2020a]{FallahConvergence}
Fallah, A., Mokhtari, A., and Ozdaglar, A. (2020a).
\newblock On the convergence theory of gradient-based model-agnostic
  meta-learning algorithms.
\newblock In {\em Proc. Twenty Third International Conference on Artificial
  Intelligence and Statistics}, volume 108, pages 1082--1092.

\bibitem[Fallah et~al., 2020b]{Fallah2020Personalized}
Fallah, A., Mokhtari, A., and Ozdaglar, A. (2020b).
\newblock Personalized federated learning: A meta-learning approach.
\newblock {\em {Available as} arXiv: 2002.07948}.

\bibitem[Finn et~al., 2017]{FinnAL17}
Finn, C., Abbeel, P., and Levine, S. (2017).
\newblock Model-agnostic meta-learning for fast adaptation of deep networks.
\newblock In {\em Proc. International Conference on Machine Learning}, pages
  1126--1135, Sydney, Australia.

\bibitem[Finn et~al., 2018]{ProbabilisticMAML}
Finn, C., Xu, K., and Levine, S. (2018).
\newblock Probabilistic model-agnostic meta-learning.
\newblock In {\em Proc. 32nd International Conference on Neural Information
  Processing Systems}, page 9537–9548, Montr\'{e}al, Canada.

\bibitem[Hospedales et~al., 2020]{Hospedales2020MetaLearningIN}
Hospedales, T.~M., Antoniou, A., Micaelli, P., and Storkey, A.~J. (2020).
\newblock Meta-learning in neural networks: A survey.
\newblock {\em {Available as} arXiv:2004.05439}.

\bibitem[Ioffe and Szegedy, 2015]{Batch_Normalization}
Ioffe, S. and Szegedy, C. (2015).
\newblock Batch normalization: Accelerating deep network training by reducing
  internal covariate shift.
\newblock In {\em Proc. 32nd International Conference on International
  Conference on Machine Learning}, volume~37, page 448–456.

\bibitem[Ji et~al., 2020]{Ji2020MultiStepMM}
Ji, K., Yang, J., and Liang, Y. (2020).
\newblock Theoretical convergence of multi-step model-agnostic meta-learning.
\newblock {\em {Available as} arXiv:2002.07836}.

\bibitem[Jiang et~al., 2019]{ImprovingFL}
Jiang, Y., Konecn{\'y}, J., Rush, K., and Kannan, S. (2019).
\newblock Improving federated learning personalization via model agnostic meta
  learning.
\newblock {\em {Available as} arXiv:1909.12488}.

\bibitem[Khodak et~al., 2019]{BalcanAdaptiveGradient}
Khodak, M., Balcan, M.-F., and Talwalkar, A. (2019).
\newblock Adaptive gradient-based meta-learning methods.
\newblock In {\em Advances in Neural Information Processing Systems 32}, pages
  5917--5928. Vancouver, Canada.

\bibitem[Kingma and Ba, 2015]{Adam}
Kingma, D.~P. and Ba, J. (2015).
\newblock Adam: {A} method for stochastic optimization.
\newblock In {\em Proc. International Conference on Learning Representations},
  San Diego, CA, USA.

\bibitem[Koch et~al., 2015]{koch2015siamese}
Koch, G., Zemel, R., and Salakhutdinov, R. (2015).
\newblock Siamese neural networks for one-shot image recognition.
\newblock In {\em Proc. ICML Deep Learning Workshop}, Lille, France.

\bibitem[Lake et~al., 2015]{Omniglot}
Lake, B.~M., Salakhutdinov, R., and Tenenbaum, J.~B. (2015).
\newblock Human-level concept learning through probabilistic program induction.
\newblock {\em Science}, 350(6266):1332--1338.

\bibitem[Li et~al., 2017]{metasgd}
Li, Z., Zhou, F., Chen, F., and Li, H. (2017).
\newblock {Meta-SGD}: Learning to learn quickly for few shot learning.
\newblock {\em {Available as} arXiv: 1707.09835}.

\bibitem[Lian et~al., 2017]{Lian17}
Lian, X., Zhang, C., Zhang, H., Hsieh, C.-J., Zhang, W., and Liu, J. (2017).
\newblock Can decentralized algorithms outperform centralized algorithms? {A}
  case study for decentralized parallel stochastic gradient descent.
\newblock In {\em Advances in Neural Information Processing Systems 30}, pages
  5330--5340. Long Beach, CA, USA.

\bibitem[Nassif et~al., 2016]{Nassif16}
Nassif, R., Richard, C., Ferrari, A., and Sayed, A.~H. (2016).
\newblock Proximal multitask learning over networks with sparsity-inducing
  coregularization.
\newblock {\em IEEE Transactions on Signal Processing}, 64(23):6329--6344.

\bibitem[Nedic and Ozdaglar, 2009]{Nedic09}
Nedic, A. and Ozdaglar, A. (2009).
\newblock {Distributed subgradient methods for multi-agent optimization}.
\newblock {\em IEEE Trans. Automatic Control}, 54(1):48--61.

\bibitem[Nichol et~al., 2018]{reptile}
Nichol, A., Achiam, J., and Schulman, J. (2018).
\newblock On first-order meta-learning algorithms.
\newblock {\em {Available as} arXiv: 1803.02999}.

\bibitem[Raghu et~al., 2020]{Raghu2020Rapid}
Raghu, A., Raghu, M., Bengio, S., and Vinyals, O. (2020).
\newblock Rapid learning or feature reuse? {T}owards understanding the
  effectiveness of {MAML}.
\newblock In {\em Proc. International Conference on Learning Representations}.

\bibitem[Ravi and Larochelle, 2017]{Ravi17}
Ravi, S. and Larochelle, H. (2017).
\newblock Optimization as a model for few-shot learning.
\newblock In {\em Proc. International Conference on Learning Representations},
  Toulon, France.

\bibitem[Sahin, 2004]{Sahin05}
Sahin, E. (2004).
\newblock Swarm robotics: From sources of inspiration to domains of
  application.
\newblock In {\em Proc. International Workshop on Swarm Robotics}, pages
  10--20, Santa Monica, CA, USA.

\bibitem[Santoro et~al., 2016]{Santoro16}
Santoro, A., Bartunov, S., Botvinick, M., Wierstra, D., and Lillicrap, T.
  (2016).
\newblock Meta-learning with memory-augmented neural networks.
\newblock In {\em Proc. 33rd International Conference on International
  Conference on Machine Learning}, volume~48, page 1842–1850, New York, NY,
  USA.

\bibitem[Sayed, 2014a]{Sayed14}
Sayed, A.~H. (2014a).
\newblock {Adaptation, learning, and optimization over networks}.
\newblock {\em Foundations and Trends in Machine Learning}, 7(4-5):311--801.

\bibitem[Sayed, 2014b]{Sayed14proc}
Sayed, A.~H. (2014b).
\newblock {Adaptive networks}.
\newblock {\em Proc. of the IEEE}, 102(4):460--497.

\bibitem[Schmidhuber, 1987]{sch87}
Schmidhuber, J. (1987).
\newblock Evolutionary principles in self-referential learning. on learning how
  to learn: The meta-meta-meta...-hook.
\newblock {Diploma Thesis}, Technische Universitat Munchen, Germany.

\bibitem[Schmidhuber, 1992]{sch92}
Schmidhuber, J. (1992).
\newblock Learning to control fast-weight memories: An alternative to dynamic
  recurrent networks.
\newblock {\em Neural Computation}, 4(1):131–139.

\bibitem[Vinyals et~al., 2016]{VinyalsMatchingNetworks}
Vinyals, O., Blundell, C., Lillicrap, T., Kavukcuoglu, K., and Wierstra, D.
  (2016).
\newblock Matching networks for one shot learning.
\newblock In {\em Advances in Neural Information Processing Systems 29}, pages
  3630--3638. Barcelona, Spain.

\bibitem[Vlaski and Sayed, 2019a]{Vlaski19}
Vlaski, S. and Sayed, A.~H. (2019a).
\newblock Diffusion learning in non-convex environments.
\newblock In {\em Proc. of IEEE ICASSP}, pages 5262--5266, Brighton, UK.

\bibitem[Vlaski and Sayed, 2019b]{Vlaski19nonconvexP1}
Vlaski, S. and Sayed, A.~H. (2019b).
\newblock {Distributed learning in non-convex environments -- Part I: Agreement
  at a linear rate}.
\newblock {\em {Available as} arXiv: 1907.01848}.

\bibitem[Vlaski and Sayed, 2020]{Vlaski2020SecondOrder}
Vlaski, S. and Sayed, A.~H. (2020).
\newblock Second-order guarantees in centralized, federated and decentralized
  nonconvex optimization.
\newblock {\em to appear. {Available as} arXiv:2003.14366}.

\bibitem[Xiao and Boyd, 2003]{Xiao03}
Xiao, L. and Boyd, S. (2003).
\newblock Fast linear iterations for distributed averaging.
\newblock In {\em Proc. 42nd IEEE International Conference on Decision and
  Control}, volume~5, pages 4997--5002.

\bibitem[Yuan et~al., 2016]{Yuan16}
Yuan, K., Ling, Q., and Yin, W. (2016).
\newblock On the convergence of decentralized gradient descent.
\newblock {\em SIAM Journal on Optimization}, 26(3):1835--1854.

\bibitem[Zhang et~al., 2019]{clinicalriskpred}
Zhang, X.~S., Tang, F., Dodge, H.~H., Zhou, J., and Wang, F. (2019).
\newblock {MetaPred}: Meta-learning for clinical risk prediction with limited
  patient electronic health records.
\newblock In {\em Proc. 25th ACM SIGKDD International Conference on Knowledge
  Discovery and Data Mining}, page 2487–2495.

\bibitem[{Zhuang} et~al., 2020]{Nonconvex_online_ml}
{Zhuang}, Z., {Wang}, Y., {Yu}, K., and {Lu}, S. (2020).
\newblock No-regret non-convex online meta-learning.
\newblock In {\em Proc. IEEE ICASSP}, pages 3942--3946.

\end{thebibliography}
